\documentclass{article}

\usepackage{arxiv}

\usepackage{amsmath,amssymb,amsfonts,amsthm,mathtools}

\usepackage[utf8]{inputenc}
\usepackage[T1]{fontenc}
\usepackage{hyperref}
\usepackage{url}
\usepackage{booktabs}
\usepackage{array}
\usepackage{nicefrac}
\usepackage{microtype}
\usepackage{xcolor}
\usepackage{graphicx}
\usepackage{float}
\usepackage{multirow}
\usepackage{makecell}
\usepackage{algorithm}
\usepackage{algorithmic}
\usepackage{cite}

\usepackage[capitalize,noabbrev]{cleveref}

\providecommand{\IEEEPARstart}[2]{#1#2}

\theoremstyle{plain}
\newtheorem{theorem}{Theorem}
\newtheorem{proposition}[theorem]{Proposition}
\newtheorem{lemma}[theorem]{Lemma}
\newtheorem{corollary}[theorem]{Corollary}
\theoremstyle{definition}

\newtheorem{assumption}{Assumption}
\theoremstyle{remark}
\newtheorem{remark}[theorem]{Remark}

\crefname{assumption}{Assumption}{Assumptions}
\Crefname{assumption}{Assumption}{Assumptions}

\newcommand{\E}{\mathbb{E}}
\newcommand{\Prob}{\mathbb{P}}
\newcommand{\Ind}{\mathbf{1}}

\DeclareMathOperator*{\argmax}{arg\,max}

\newcommand{\Dtr}{D_{\mathrm{tr}}}
\newcommand{\Dtune}{D_{\mathrm{tune}}}
\newcommand{\Dcert}{D_{\mathrm{cert}}}

\newcommand{\ncert}{n_{\mathrm{cert}}}
\newcommand{\grid}{\Lambda \times T}
\newcommand{\Hset}{\mathcal{H}}
\newcommand{\Ghat}{\widehat{\mathcal{G}}}

\newcommand{\Rsel}{R_{\mathrm{sel}}}
\newcommand{\pacc}{p_{\mathrm{acc}}}
\newcommand{\pmin}{\pi_{\min}}
\newcommand{\Udep}{U_{\mathrm{dep}}}
\newcommand{\udep}{u^{\mathrm{dep}}}
\newcommand{\Umargin}{U^{*,\mathrm{margin}}_{\mathrm{dep}}}
\newcommand{\Umarginva}{U^{*,\mathrm{margin}\text{-}\mathrm{va}}_{\mathrm{dep}}}
\newcommand{\Mset}{M}

\newcommand{\paccHat}{\widehat{p}_{\mathrm{acc}}}
\newcommand{\RselHat}{\widehat{R}_{\mathrm{sel}}}
\newcommand{\Zbar}{\overline{Z}}
\newcommand{\ubar}{\overline{u}}
\newcommand{\SigmaZ}{\widehat{\sigma}_Z^2}
\newcommand{\SigmaU}{\widehat{\sigma}_u^2}
\newcommand{\sigmaHat}{\widehat{\sigma}^2}
\newcommand{\etaZ}{\eta_Z}
\newcommand{\etaU}{\eta_u}
\newcommand{\gammar}{\gamma_r}
\newcommand{\gammau}{\gamma_u}

\newcommand{\EB}{\mathrm{EB}}
\newcommand{\pLCB}{p_{\mathrm{LCB}}}
\newcommand{\ULCB}{U_{\mathrm{LCB}}}
\newcommand{\UCBours}{\mathrm{UCB}^{\Rsel}_{\mathrm{Ours}}}
\newcommand{\UCBHoeff}{\mathrm{UCB}^{\Rsel}_{\mathrm{Hoeff}}}

\newcommand{\LO}{L_O}
\newcommand{\LH}{L_H}
\newcommand{\Tobs}{T_{\mathrm{obs}}}
\newcommand{\Tex}{T_{\mathrm{ex}}}
\newcommand{\sigmastar}{\sigma^2_{\!*}}
\newcommand{\sigmaexact}{\sigma^2_{\mathrm{ex}}}
\newcommand{\sznot}{s_0}

\newcommand{\SCoRC}{\textsc{SCoRC}}
\newcommand{\HCRC}{Hoeffding\nobreakdash--CRC}

\newcommand{\PASS}{\checkmark}
\newcommand{\FAIL}{$\times$}

\renewcommand{\shorttitle}{\SCoRC{}: A Joint Finite-Sample Certificate}

\title{A Joint Finite-Sample Certificate for Adaptive Selective Conformal Risk Control}

\author{%
  Xiaoli Yu \\
  School of Cyber Security and Information Law \\
  Chongqing University of Posts and Telecommunications \\
  Chongqing 400065, China
  \And
  Jiamiao Liu\thanks{Corresponding author.} \\
  Department of Information, Xinqiao Hospital \\
  Army Medical University (Third Military Medical University) \\
  Chongqing 400037, China \\
  \texttt{ljm6046@163.com}
}

\date{}

\begin{document}

\maketitle

\begin{abstract}
Selective predictors answer on confident inputs and abstain elsewhere; deploying one safely needs a single finite-sample certificate that simultaneously upper-bounds the selected risk, lower-bounds the acceptance probability $\pacc$ above a floor $\pmin$, and lower-bounds the deployment utility. This certificate must be valid under adaptive threshold selection from a finite grid of $m$ pairs on $\ncert$ samples. We give such a certificate for bounded, possibly non-monotone losses by treating the selected risk directly as a ratio rather than through a Hoeffding-style range bound. The construction couples three confidence bounds: a variance-adaptive empirical-Bernstein bound on the ratio risk, a Clopper--Pearson bound on acceptance, and a two-sided closeness bound on utility. Together they lower-bound the certified policy's utility absolutely and to within $2\gammau$ of the best over the \emph{certified set}, both non-vacuous whenever feasible; a regime-scoped third leg matches an external oracle, informative only where the risk margin $\gammar < \alpha$ and vacuous at the headline operating points. Relative to the range-only Hoeffding-ratio construction this sharpens the acceptance-floor dependence from $1/\pmin$ to $1/\sqrt{\pmin}$, and a closed-form corollary identifies a per-pair regime in which our risk bound dominates a Hoeffding conformal risk control (Hoeffding--CRC) selective bound. Empirically, on ImageNet (three ResNets) and COCO val 2017 panoptic, the certificate opens a $+22$ pp certified-acceptance frontier over Hoeffding--CRC and is ${\approx}10{\times}$ tighter than a non-vacuous matched-valid baseline; these gains are regime-scoped, not universal, and absent on ADE20K. The certifier runs in $O(\ncert m)$ time.

\end{abstract}


\section{Introduction}
\label{sec:introduction}

\IEEEPARstart{S}{elective} predictors abstain on some inputs and answer on the rest~\cite{chow1970_reject, elyaniv2010_selective, geifman2019_selectivenet}. In safety-critical pattern analysis (triaging chest-radiograph reports for radiologist review, abstaining from steering control under perception ambiguity, and routing image content to human moderators), reliable acceptance alone is not enough. The operator must also know how often the system will answer, and whether abstention is utility-justified once the costs of deferral are weighed against accepted decisions. A deployable selective system therefore needs a single finite-sample certificate covering three population quantities at once, not a high-probability bound on one of them.

We target the joint object
\[
\Rsel = \frac{\E[A L]}{\E[A]}, \quad \pacc = \E[A], \quad \Udep = \E[A v - c (1 - A)],
\]
the selected risk on accepted outputs, the acceptance probability, and the marginal deployment utility, where the acceptance indicator $A$ depends on a risk threshold and an abstention threshold chosen adaptively from a finite calibration grid $\grid$. No prior distribution-free certificate covers $(\Rsel, \pacc, \Udep)$ simultaneously after adaptive selection from $\grid$: conformal risk control (CRC) bounds only a non-selective expected loss~\cite{angelopoulos2022_crc, bates2021_rcps}; selective conformal procedures certify either a binary selective error or an asymptotic e-value version of utility~\cite{xu2025_scrc, bai2026_score}; per-pair Hoeffding--CRC variants give pointwise bounds that do not jointly cover acceptance and utility. \cref{tab:positioning} (\cref{sec:related}) tabulates the gap.

We close this gap. \cref{thm:joint-cert} gives a joint $(1 - \delta)$ finite-sample certificate for bounded, possibly non-monotone losses~\cite{angelopoulos2026_nmcrc, aldirawi2026_nmcrc} under adaptive two-threshold selection from $\grid$. The certified pair $(\hat\lambda, \hat\tau)$ satisfies $\Rsel \le \alpha$ and $\pacc \ge \pmin$ with probability at least $1 - \delta$; on the same event, $\Udep$ at the certified pair admits a finite-sample lower bound $\Udep(\hat\lambda, \hat\tau) \ge \ULCB(\hat\lambda, \hat\tau)$ and lies within $2\gammau$ of the best deployment utility over the certified set (\cref{cor:gset-opt}), both non-vacuous whenever the certifier is feasible. A stronger \emph{external} optimality holds against an oracle restricted to risk-certifiable policies ($\Rsel \le \alpha - \gammar$, $\pacc \ge 2\pmin$); this margin-oracle leg is informative in the regime $\gammar < \alpha$ and is vacuously valid otherwise. The factor of two between the certified floor $\pacc \ge \pmin$ and the oracle margin $\pacc \ge 2 \pmin$ is a single derived consequence of finite-sample inversion. Assumptions are light: i.i.d.\ data under a three-split protocol, with the grid and user parameters $(\alpha, \pmin, \delta)$ fixed before certification and a bounded loss.

The certificate is assembled from three coupled confidence bounds computed in a single pass over the certification split (\cref{fig:hero}): a variance-adaptive empirical-Bernstein upper bound~\cite{howard2021_csequences} on the selected risk, a Clopper--Pearson lower bound on $\pacc$, and a two-sided Maurer--Pontil bound on $\Udep$. Relative to the standard range-only Hoeffding-ratio construction, this changes the $\Rsel$-margin acceptance-floor dependence from $1/\pmin$ to $1/\sqrt{\pmin}$; a \emph{construction-specific} lower bound (\cref{prop:lower-bound}), exhibited on a single two-point distribution, shows that the range-only Hoeffding-ratio construction cannot escape the $1/\pmin$ rate, so the separation is genuine \emph{for that construction}, not a minimax separation over the certificate class, which is left open.

\begin{figure*}[t]
\centering
\includegraphics[width=\textwidth]{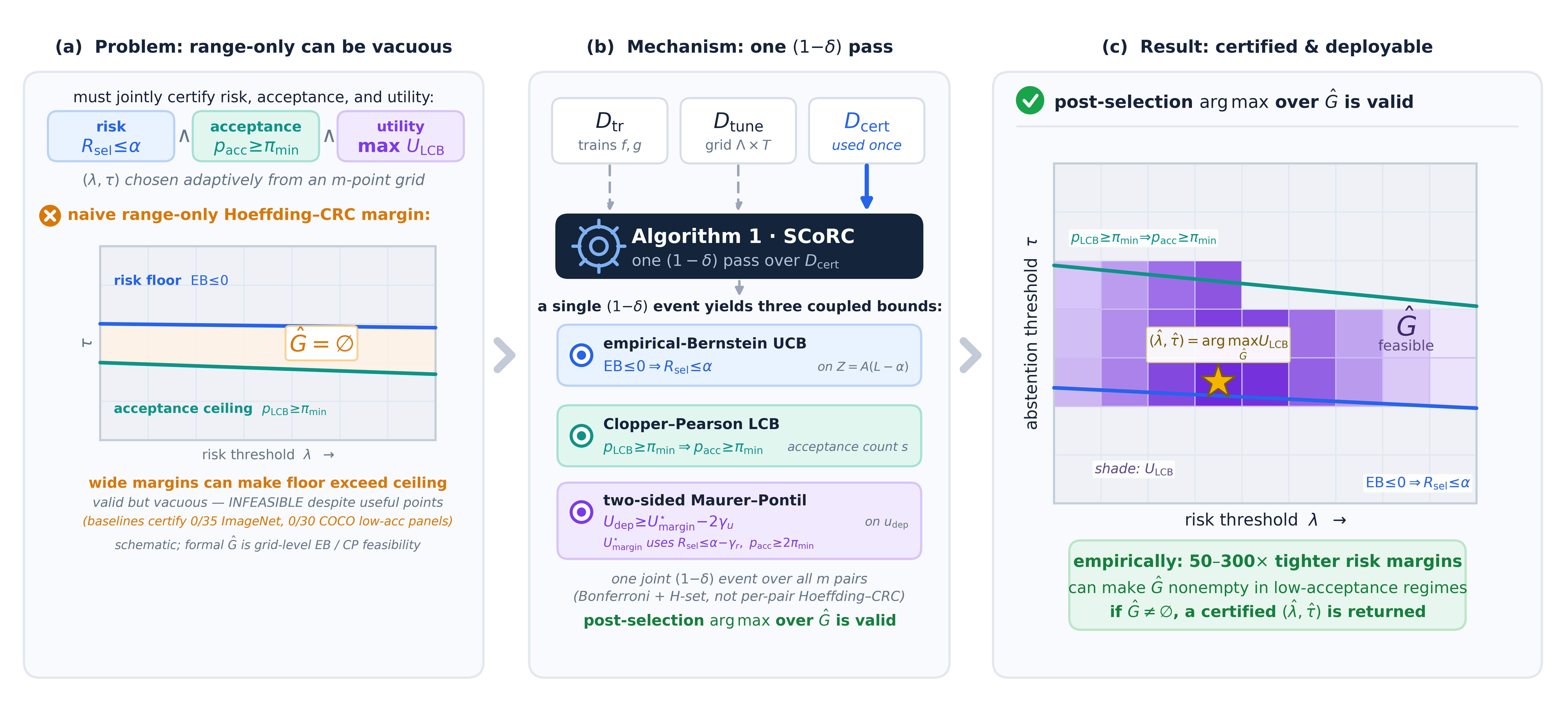}
\caption[The joint certificate at a glance: problem, mechanism, and result.]{The joint certificate at a glance: \textbf{problem}, \textbf{mechanism}, and \textbf{result}.
\textbf{(a) Problem.} A deployable selective predictor must \emph{jointly} certify selected risk ($\Rsel \le \alpha$), acceptance ($\pacc \ge \pmin$), and utility ($\Udep$). Under the range-only Hoeffding-ratio construction the risk margin is wide enough that the certifiable risk floor can rise above the acceptance ceiling, so no operating point is jointly certifiable and the certificate is vacuous (schematic).
\textbf{(b) Mechanism.} A three-split protocol ($\Dtr$ trains the base predictor $f$ and acceptance score $g$; $\Dtune$ fixes the candidate grid $\grid$ and the user parameters $(\alpha, \pmin, \delta)$; $\Dcert$ is consumed once) feeds \cref{alg:certify} (\SCoRC), which makes a single $(1 - \delta)$ pass over $\Dcert$ producing three coupled bounds: an empirical-Bernstein upper bound on $Z = A(L - \alpha)$ (controlling $\Rsel$), a Clopper--Pearson lower bound on $\pacc$, and a two-sided Maurer--Pontil bound on $\Udep$. One joint $(1 - \delta)$ event over all $m$ grid pairs (a deterministic $\Hset$-set union, not a per-pair \HCRC{} construction) makes the post-selection $\argmax$ over the certified feasible set $\Ghat$ valid.
\textbf{(c) Result.} When $\Ghat \ne \emptyset$ the certifier returns the pair $(\hat\lambda, \hat\tau) = \argmax_{(\lambda, \tau) \in \Ghat} \ULCB$ inside the feasible region of the risk/abstention-threshold plane, and \textsc{Infeasible} otherwise. Empirically, $50$--$300{\times}$ tighter risk margins can make $\Ghat$ nonempty in low-acceptance regimes that the range-only baseline cannot certify.}
\label{fig:hero}
\end{figure*}

We validate the certificate on two task families. On ImageNet classification with three ResNet backbones, the joint construction certifies low-acceptance operating points the textbook range-only Hoeffding--CRC baseline cannot reach (\cref{fig:cert-frontier}); since that baseline is vacuous at this floor, the $50$--$300{\times}$ all-pair width ratio measures vacuity-avoidance, while against a non-vacuous matched-valid normalisation the certificate is ${\approx}\,10{\times}$ tighter (tying at maximum acceptance). Tighter per-pair primitives that certify only the narrower per-pair risk object (Bernstein, WSR) can win in their own regimes (\cref{sec:discussion}). On COCO val 2017 panoptic segmentation under pixel-accuracy loss, it opens a $+22$\,pp certified-acceptance frontier above Hoeffding--CRC. On ADE20K with a binary mIoU loss the joint construction is not amortised, and a per-pair Hoeffding--CRC comparator dominates: the payoff is regime-characterised, not universal. A closed-form corollary (\cref{cor:regime}) predicts which side of the comparison applies on a given pair from the accepted-sample variance, mean, and acceptance fraction.

\noindent\textbf{Contributions.}
\begin{enumerate}
\item \textbf{Joint certificate (\cref{thm:joint-cert}).} A finite-sample $(1 - \delta)$ certificate for selective prediction that simultaneously bounds $\Rsel \le \alpha$, $\pacc \ge \pmin$, and $\Udep$, under bounded non-monotone losses and adaptive two-threshold selection on a finite grid, with direct ratio handling of $\Rsel$. The utility leg holds as an absolute lower bound $\Udep \ge \ULCB$ and a certified-set optimality guarantee (\cref{cor:gset-opt}), both always non-vacuous when feasible. It additionally provides a regime-scoped external margin-oracle bound that is informative when $\gammar < \alpha$, with a variance-adaptive refinement in \cref{cor:va-oracle}.

\item \textbf{Variance-adaptive post-selection construction.} A coupled empirical-Bernstein / Clopper--Pearson / Maurer--Pontil composition with a deterministic-eligibility ($\Hset$-set) union argument (\cref{lem:inclusion}) that preserves the $\delta/4$ per-component budget and reduces the $\Rsel$-margin acceptance-floor dependence from $1/\pmin$ to $1/\sqrt{\pmin}$ relative to the range-only Hoeffding-ratio construction (\cref{prop:lower-bound}).

\item \textbf{Regime-separation corollary (\cref{cor:regime}).} A closed-form per-pair sufficient condition on accepted-sample variance, mean, and acceptance fraction under which the proposed $\Rsel$ upper bound dominates the Hoeffding--CRC selective upper bound on the same pair; the statement is scoped to that comparison and does not assert universal dominance.

\item \textbf{Algorithm and empirical validation.} An $O(\ncert m)$ calibration-time certifier (\cref{alg:certify}) with an explicit \textsc{Infeasible} return, evaluated on ImageNet classification (three ResNet backbones) and dense pattern analysis (COCO val 2017 panoptic, ADE20K), delivering a $+22$\,pp certified-acceptance frontier on COCO and ${\approx}\,10{\times}$ tighter per-pair risk bounds than a non-vacuous matched-valid normalisation on ImageNet (ties at maximum acceptance; $50$--$300{\times}$ versus the textbook range-only Hoeffding--CRC baseline, which is vacuous here), with the ADE20K result delimiting the regime of payoff.
\end{enumerate}

\noindent\textbf{Organisation.} \cref{sec:related} positions the certificate against prior CRC, RCPS, and selective-prediction work; \cref{sec:method} states the three-split protocol and \cref{alg:certify}; \cref{sec:theory} proves \cref{thm:joint-cert} and \cref{cor:regime}; \cref{sec:experiments} reports the empirical payoff and the regime characterisation; \cref{sec:discussion} discusses scope and open directions, and \cref{sec:conclusion} closes.

\section{Related Work}
\label{sec:related}

Six axes separate methods for selective deployment under bounded loss. \cref{tab:positioning} tabulates how each line of work fares on each axis; the rest of this section describes the lines and positions ours.

\begin{table*}[t]
\centering
\caption{Six-axis positioning of selective certification methods. Cell entries reflect what each method \emph{publishes}; ``n/a'' indicates the axis is not addressed by that method's stated object; ``partial'' indicates the method addresses the axis in a restricted regime that does not subsume our setting.}
\label{tab:positioning}
\footnotesize
\renewcommand{\arraystretch}{1.15}
\resizebox{\textwidth}{!}{%
\begin{tabular}{lcccccc}
\toprule
Method & \makecell{Certifies \\ $\Rsel$} & \makecell{Certifies \\ $\pacc$} & \makecell{Finite-sample \\ utility LCB} & \makecell{Adaptive \\ $(\lambda, \tau)$ grid} & \makecell{Direct ratio \\ $\E[A L]/\E[A]$} & \makecell{Non-monotone \\ loss} \\
\midrule
\textbf{Ours (\SCoRC)}                                      & \PASS                & \PASS                & \PASS         & \PASS         & \PASS        & \PASS         \\
SCRC \cite{xu2025_scrc}                                     & \PASS                & implicit cutoff      & \FAIL         & joint $(\lambda_1, \lambda_2)$ & avoids ratio  & needs monotone \\
Conformal selective prediction \cite{bai2026_score}         & \PASS                & \FAIL                & population-asymp. & \FAIL         & partial      & unclear       \\
LTT \cite{angelopoulos2021_ltt}                             & n/a                  & n/a                  & \FAIL         & partial       & \FAIL        & \PASS         \\
Non-monotone CRC \cite{angelopoulos2026_nmcrc,aldirawi2026_nmcrc} & non-selective   & n/a                  & \FAIL         & \FAIL         & \FAIL        & \PASS         \\
CRC \cite{angelopoulos2022_crc}, RCPS \cite{bates2021_rcps} & n/a (non-selective)  & n/a                  & \FAIL         & \FAIL         & n/a          & needs monotone \\
WSR \cite{waudby2024_wsr}                                   & per-pair             & \FAIL                & \FAIL         & regime-dep.   & via $A(L - \alpha)$ & \PASS  \\
\bottomrule
\end{tabular}%
}
\end{table*}

\subsection{Selective classification and the reject option}

Selective classification with a reject option goes back to Chow~\cite{chow1970_reject}, who derived the optimal error--reject tradeoff for a known likelihood ratio. El-Yaniv and Wiener~\cite{elyaniv2010_selective} laid the noise-free foundations of selective classification in a learning-theoretic framework; Geifman and El-Yaniv~\cite{geifman2017_selective_dnn} gave a high-probability bound on the selective risk for deep networks via a guaranteed-risk selection rule, and the same authors~\cite{geifman2019_selectivenet} introduced the SelectiveNet architecture with an integrated reject option. These works control the selective risk (or its risk--coverage tradeoff) for a \emph{single} pre-specified target, but do not deliver a joint finite-sample certificate that simultaneously covers the selected risk, the acceptance probability, and the deployment utility under adaptive two-threshold grid selection. Our work picks up the certified-deployment question in this lineage and gives the joint object.

\subsection{Conformal prediction and risk control}

Conformal prediction \cite{bates2021_rcps,angelopoulos2021_ltt} delivers distribution-free coverage on prediction sets. Conformal risk control \cite{angelopoulos2022_crc} and its risk-controlling-prediction-sets predecessor \cite{bates2021_rcps} extend the framework to bounded expected losses. These approaches certify the \emph{non-selective} expected loss $\E[L]$ and sidestep the random denominator $\E[A]$ that makes the ratio-form selected risk $\Rsel$ the harder object. Cross-validated and ordinal variants \cite{cohen2024_cvcrc,xu2024_crc_ordinal} sharpen rates and broaden the loss class but inherit the non-selective object. Pareto testing \cite{laufer2022_pareto} controls multiple risks simultaneously by reducing to a sequence of conformalised tests; its post-selection guarantee is over a different ordering of operating points than ours and assumes a fixed Pareto frontier.

\subsection{Selective conformal risk control}

Selective CRC \cite{xu2025_scrc} handles a joint $(\lambda_1, \lambda_2)$ selection but requires the loss to be non-increasing in $\lambda_2$ and uses a population-quantile cutoff to sidestep the random denominator. The recent selective-prediction work of \cite{bai2026_score} introduces an e-value certificate via $\E[L \cdot E] \le 1$ that avoids the ratio directly but gives a single trust threshold with population-asymptotic Neyman--Pearson utility; no finite-sample utility lower bound is established and no separate acceptance lower bound is delivered. Two-stage risk control for ranked retrieval \cite{xu2024_twostage} controls risk across a sequential retrieval-then-ranking pipeline rather than under adaptive threshold-pair selection, while conformal selective inference for false coverage \cite{bao2024_scop,jin2024_focal} treats selection of \emph{test points} rather than of threshold pairs, both related but distinct problems. None of these works delivers a joint certificate on $(\Rsel, \pacc, \Udep)$ under adaptive grid selection.

\subsection{Non-monotone and adaptive CRC}

Algorithmic-stability-based non-monotone CRC \cite{angelopoulos2026_nmcrc} certifies bounded non-monotone losses with multidimensional parameters under an algorithmic-stability assumption, with selective image classification as one demonstration application; its certificate target is the (non-selective) expected loss, not the joint $(\Rsel, \pacc, \Udep)$ object with direct ratio handling of the selected risk under adaptive grid selection. A concurrent line \cite{aldirawi2026_nmcrc} establishes finite-sample guarantees for non-monotone CRC but does not give a joint certificate on $(\Rsel, \pacc, \Udep)$ under adaptive grid selection. Automatically adaptive CRC \cite{blot2024_aacrc} achieves approximate \emph{conditional} risk control by data-driven adaptation to test-sample difficulty, but targets the non-selective expected loss and lacks the joint $(\lambda, \tau)$ coupling on the selected-risk object. Online conformal abstention \cite{online2025_conformal_abstention} controls factuality via abstention under adversarial bandit (partial) feedback in interactive language-model systems, a different online setting; we cite it to disambiguate the term ``selective''. Anytime-valid conformal risk control \cite{hultberg2026_avcrc} gives high-probability anytime risk control but not the joint selective object. We treat these as the closest near-misses to disambiguate our claim of being the first to deliver the joint certificate under adaptive selection.

\subsection{Concentration primitives and comparator baselines}

We build on the empirical-Bernstein inequality \cite{howard2021_csequences} (used in two-sided form via the conservative two-tail union in our deployed code path) and Clopper--Pearson binomial confidence intervals as black boxes. As comparator primitives we test against the predictable-mixture betting confidence sequence \cite{waudby2024_wsr}, e-value post-selection corrections \cite{xu2022_psi_evalue_ci}, the standard Hoeffding-on-range baseline, and a Bernstein-on-accepted-samples baseline that is the structural sibling of our method without joint-certification coupling. Time-uniform nonparametric confidence sequences \cite{howard2021_csequences} cover the anytime-valid setting and inform the WSR primitive.

\subsection{Post-selection inference}

Our adaptive-grid argument is sometimes confused with naive uniform concentration plus a union bound over the grid. Two distinctions matter, both made precise in \cref{sec:theory}. First, the deterministic-eligibility restriction in our inclusion lemma (\cref{lem:inclusion}) is the technical step where naive union bounding fails: confining the acceptance-count concentration events to the deterministic eligible set keeps the per-component failure budget from inflating, whereas an unrestricted union does not; the delta ledger (\cref{tab:delta-ledger}) itemises every contribution. Second, at low acceptance our construction is tighter than the post-selection correction a uniform union argument would give: a variance-adaptive sharpening of the acceptance-floor dependence that \cref{thm:joint-cert} quantifies and that \cref{prop:lower-bound} shows the range-only alternative cannot match. Recent post-selection conformal work on e-value confidence intervals \cite{xu2022_psi_evalue_ci} handles the correction at the e-value level but does not address the simultaneous $(\Rsel, \pacc, \Udep)$ object.

\subsection{Distribution shift and beyond i.i.d.}

Non-exchangeable conformal risk control \cite{farinhas2023_nexcrc} and weighted variants \cite{zecchin2025_wcrc} extend CRC beyond i.i.d.\ calibration but for non-selective objects. The semi-supervised \cite{einbinder2024_sscalib} and multiply robust \cite{paul2025_mrcrc} extensions sharpen rates under additional structure. We do not claim shift robustness in our finite-sample theory; we report a descriptive ImageNet-V2 evaluation in \cref{sec:experiments} as a stress test only and document why the sample-size condition is not met. Future extensions to weighted and non-exchangeable selective settings are natural and discussed in \cref{sec:conclusion}.

\paragraph{Positioning.} Each line of prior work above addresses some but not all of four ingredients: a joint certificate on $\Rsel + \pacc + \Udep$ rather than on one or two of them; adaptive selection over a finite $(\lambda, \tau)$ grid; direct handling of the random denominator in the ratio-form selected risk $\Rsel$; and a bounded non-monotone loss. \Cref{tab:positioning} tabulates which combination each method covers. No published method currently provides simultaneous coverage of all four. Cell content in \cref{tab:positioning} reflects what each method \emph{publishes}; we do not claim other methods could not be extended to cover the missing cells, only that they currently do not.

\section{Problem, Setup, and Algorithm}
\label{sec:method}

\subsection{Problem and notation}
\label{sec:notation}

Let $(X_i, Y_i)_{i=1}^{\ncert}$ be i.i.d.\ samples from an unknown distribution $P_{XY}$. A base predictor $f$ produces a candidate output $\hat C_\lambda(X_i)$ controlled by a risk threshold $\lambda \in \Lambda$, and an acceptance score $g(X_i)$ controls a $\{0, 1\}$-valued acceptance indicator $A(\lambda, \tau)(X_i) := \Ind\{g(X_i) > \tau\}$ at an abstention threshold $\tau \in T$. The grid $\grid$ has $|\grid| = m$ candidate pairs. The bounded loss $L : \mathcal X \times \mathcal Y \to [0, B]$ may be non-monotone in $\lambda$; the deployment value $v : \mathcal X \times \mathcal Y \to [0, V]$ is bounded; the abstention cost $c \ge 0$ is known. Define per-sample contributions $Z_i(\lambda, \tau) := A_i(\lambda, \tau) (L_i(\lambda, \tau) - \alpha)$ and $u_i^{\mathrm{dep}}(\lambda, \tau) := A_i(\lambda, \tau) v_i - c (1 - A_i(\lambda, \tau))$, and let $\Zbar, \ubar, \SigmaZ, \SigmaU, s(\lambda, \tau) := \sum_i A_i(\lambda, \tau)$ denote the empirical mean, Bessel-corrected sample variance, and acceptance count. The accepted-sample risk estimate and within-accepted Bessel-corrected variance, used in \cref{cor:regime}, are $\RselHat(\lambda, \tau) := \frac{1}{s} \sum_{i : A_i = 1} L_i$ and $\sigmaHat(\lambda, \tau) := \frac{1}{s - 1} \sum_{i : A_i = 1} (L_i - \RselHat)^2$, both well-defined when $s(\lambda, \tau) \ge 2$.

\subsection{The three deployment quantities}
\label{sec:deployment-quantities}

The certificate guarantees the simultaneous behaviour of three population quantities that together summarise a selective system in production. The \emph{selected risk} is defined for every pair with positive acceptance: when $\pacc(\lambda, \tau) > 0$,
$
\Rsel(\lambda, \tau) := \E[L \mid A(\lambda, \tau) = 1];
$
when $\pacc(\lambda, \tau) = 0$, $\Rsel(\lambda, \tau)$ is left undefined and no statement involving $\Rsel(\lambda, \tau)$ is asserted at that pair. All theorems and lemmas below invoke $\Rsel$ only under $\pacc \ge \pmin > 0$, and the empirical estimator $\RselHat(\lambda, \tau)$ and the within-accepted-sample variance $\sigmaHat(\lambda, \tau)$ are used only when $s(\lambda, \tau) \ge 2$. $\Rsel$ is the population loss restricted to the accepted subset, expressed as a ratio of two expectations whose random denominator is what makes this a harder object than the unconditional risk $\E[L]$. The \emph{acceptance probability}
$
\pacc(\lambda, \tau) := \E[A(\lambda, \tau)] \in [0, 1]
$
is the fraction of inputs the system answers; a lower bound on $\pacc$ certifies that the system is doing useful work rather than trivially abstaining. The \emph{marginal deployment utility}
$
\Udep(\lambda, \tau) := \E[A v - c (1 - A)] \in [-c, V]
$
nets the value collected on accepted inputs against the cost of abstention; it is positive when the system answers more inputs of value $v$ than it abstains on, and negative when abstention cost dominates. Larger $\Udep$ is better. The three are not redundant: a small-acceptance system can have $\Rsel \le \alpha$ trivially yet a poor $\Udep$ because abstention cost dominates, and a lower bound on $\pacc$ alone does not constrain the loss conditional on acceptance. We therefore certify all three jointly.

\subsection{Three-split protocol}
\label{sec:three-split}

We require a three-split partition $\Dtr \sqcup \Dtune \sqcup \Dcert$ drawn i.i.d.\ from $P_{XY}$. The base predictor $f$ and the acceptance score $g$ are constructed from $\Dtr$ (and may additionally use $\Dtune$); the candidate grid $\grid$ and the user parameters $(\alpha, \pmin, \delta)$ are constructed from $\Dtune$; the certification split $\Dcert$ is consumed exclusively by \cref{alg:certify}. The only property the proof uses is that $f, g, \grid$, and $(\alpha, \pmin, \delta)$ are \emph{cert-independent}: measurable functions of $(\Dtr, \Dtune)$ and external randomness, but not of $\Dcert$. This makes $\Dcert$ statistically independent of the construction of $f$, $g$, and $\grid$, so the i.i.d.\ assumption applies cleanly inside \cref{alg:certify} (conditionally on $\Dtr, \Dtune$).

\subsection{Assumptions}
\label{sec:assumptions}

The certificate is proved under the following assumptions:
\begin{assumption}[i.i.d.]\label{ass:iid} $(X_i, Y_i)_{i=1}^{\ncert}$ are i.i.d.\ from $P_{XY}$.\end{assumption}
\begin{assumption}[Bounded loss and utility]\label{ass:bounded} The loss bound is normalised so that $B \ge 1$. The loss and utility satisfy $L \in [0, B]$, $v \in [0, V]$, and $c \ge 0$, with $B, V, c$ known.\end{assumption}
\begin{assumption}[Finite grid and measurability]\label{ass:grid} $|\grid| = m \ge 1$, and for every $(\lambda, \tau) \in \grid$ the maps $A(\lambda, \tau) : \mathcal X \to \{0, 1\}$, $L(\lambda, \tau) : \mathcal X \times \mathcal Y \to [0, B]$, and $v : \mathcal X \times \mathcal Y \to [0, V]$ are measurable (hence so are the products $A L$ and $A v$, and the per-sample statistics $Z_i, u^{\mathrm{dep}}_i$ are bounded measurable random variables).\end{assumption}
\begin{assumption}[User parameters]\label{ass:params} $\alpha, \pmin, \delta \in (0, 1)$.\end{assumption}
\begin{assumption}[Sample size $(\star)$]\label{ass:samplesize}
\begin{equation}
\ncert \ge 32 \log(32 m / \delta) / \pmin.
\label{eq:sample-size}
\end{equation}
\end{assumption}
\begin{assumption}[Three-split]\label{ass:threesplit} $\Dtr, \Dtune, \Dcert$ are disjoint i.i.d.\ splits of $P_{XY}$; $f, g, \grid$ are constructed using only $\Dtr$ and $\Dtune$.\end{assumption}
\begin{assumption}[Selector status]\label{ass:selector} $g$ may be any measurable function of $\Dtr$ and $\Dtune$ (in particular, fixed independently of all splits, or learned from $\Dtr$ and/or $\Dtune$); the ``certify-learned'' option (constructing $g$ from $\Dcert$) is not permitted. Consequently $g$, $\grid$, and $\Hset$ are deterministic conditional on $(\Dtr, \Dtune)$.\end{assumption}

\Cref{ass:samplesize} arises from the Chernoff-step precondition in the Clopper--Pearson relative-error inversion (\cref{sublemma:cp-inversion} below) and is the only non-trivial sample-size requirement. \cref{ass:selector} is needed because the inclusion direction of \cref{lem:inclusion} uses the deterministic eligibility of the H-set $\Hset := \{(\lambda, \tau) \in \grid : \pacc(\lambda, \tau) \ge \pmin\}$; certify-learned selectors would make $\Hset$ data-dependent and break the analysis. The H-set eligibility is therefore an \emph{operational} requirement of the algorithm, not an internal proof artefact: deployments that violate it (e.g.\ by adapting $g$ on $\Dcert$) lose the certificate's validity. The ingredient stress test in \cref{sec:ingredients} (Row~6, ``Chernoff variance bridge'') reports the empirical degradation when the underlying H-set step is removed: the certifier becomes $0/10$ feasible across the tested regime.

\subsection{Algorithm}
\label{sec:algorithm}

\begin{algorithm}[t]
\caption{\SCoRC{} certifier (calibration time).}
\label{alg:certify}
\begin{algorithmic}[1]
\REQUIRE $\Dcert = (X_i, Y_i)_{i=1}^{\ncert}$ i.i.d.; grid $\grid$ of size $m$; risk target $\alpha$; acceptance floor $\pmin$; failure probability $\delta$; value $v$; cost $c$.
\STATE \textbf{Precondition $(\star)$:} $\ncert \ge 32 \log(32 m / \delta) / \pmin$.
\FOR{each pair $(\lambda, \tau) \in \grid$}
  \STATE Compute $A_i, L_i, Z_i := A_i (L_i - \alpha), u^{\mathrm{dep}}_i := A_i v_i - c (1 - A_i)$ for $i = 1, \ldots, \ncert$.
  \STATE Compute $\Zbar, \SigmaZ, s := \sum_i A_i, \ubar, \SigmaU$.
  \STATE $\EB(\lambda, \tau) \gets \Zbar + \sqrt{2 \SigmaZ \log(64 m / \delta) / \ncert} + 7 B \log(64 m / \delta) / (3 (\ncert - 1))$.
  \STATE $\pLCB(\lambda, \tau) \gets \mathrm{Beta}^{-1}(\delta / (16 m); s, \ncert - s + 1)$ if $s \ge 1$, else $0$.
  \STATE $\ULCB(\lambda, \tau) \gets \ubar - \sqrt{2 \SigmaU \log(8 m / \delta) / \ncert} - 7 (c + V) \log(8 m / \delta) / (3 (\ncert - 1))$.
\ENDFOR
\STATE $\Ghat \gets \{(\lambda, \tau) \in \grid : \EB(\lambda, \tau) \le 0 \wedge \pLCB(\lambda, \tau) \ge \pmin\}$.
\IF{$\Ghat = \emptyset$}
  \RETURN \textsc{Infeasible}.
\ELSE
  \RETURN $(\hat \lambda, \hat \tau) \gets \argmax_{(\lambda, \tau) \in \Ghat} \ULCB(\lambda, \tau)$ \hfill (ties broken by a fixed deterministic ordering).
\ENDIF
\end{algorithmic}
\end{algorithm}

\Cref{alg:certify} computes three coupled confidence bounds in a single pass over $\Dcert$ at every grid pair, takes the deterministic-feasible set
$
\Ghat := \{(\lambda, \tau) : \EB(\lambda, \tau) \le 0 \wedge \pLCB(\lambda, \tau) \ge \pmin\},
$
and returns the argmax of the utility lower bound if $\Ghat \ne \emptyset$. \textsc{Infeasible} is returned when $\Ghat = \emptyset$. The output object is the pair $(\hat \lambda, \hat \tau) \in \grid$.

\paragraph{Complexity.} Per pair, the algorithm computes $O(\ncert)$ per-sample statistics and one Beta-quantile evaluation. Total cost is $O(\ncert m)$ flops plus $O(m)$ Beta-quantile evaluations. Empirical runtime is $17.79$ ms at ImageNet scale ($\ncert = 33{,}000$, $m = 35$), with absolute timings reported in the supplementary material.

\paragraph{Implementation.} \cref{alg:certify} is implemented as a pure-Python module; the bounds use the Bessel-corrected sample variance and the two-sided Maurer--Pontil radius in the conservative form
$
\sqrt{2 \SigmaZ \log(4 / \delta') / \ncert} + 7 B \log(4 / \delta') / (3 (\ncert - 1))
$
at per-event level $\delta'$, matching the proof and the deployed code line-for-line.

\section{Joint Certificate and Regime Separation}
\label{sec:theory}

\subsection{Main result}

\begin{theorem}[Margin-oracle joint certificate]\label{thm:joint-cert}
Suppose \cref{ass:iid,ass:bounded,ass:grid,ass:params,ass:samplesize,ass:threesplit,ass:selector} hold. Define the margin-feasible set
$
\Mset(\alpha', \pmin) := \{(\lambda, \tau) \in \grid : \Rsel(\lambda, \tau) \le \alpha' \wedge \pacc(\lambda, \tau) \ge 2 \pmin\}
$
(the second argument is the certified floor $\pmin$; membership requires the \emph{stricter} oracle-margin floor $\pacc \ge 2\pmin$, the factor of two being the derived Clopper--Pearson inversion cost below) and the margin oracle $\Umargin(\alpha, \pmin) := \sup_{(\lambda', \tau') \in \Mset(\alpha - \gammar, \pmin)} \Udep(\lambda', \tau')$, with $\sup_\emptyset := -\infty$. Then, conditional on $\Dtr$ and $\Dtune$ (so that $f, g, \grid$ and the parameters are fixed, cert-independent objects; the guarantee then holds marginally over all three splits by the tower property), with probability at least $1 - \delta$ over the draw of $\Dcert$, \cref{alg:certify} either returns \textsc{Infeasible} or returns a pair $(\hat \lambda, \hat \tau) \in \Ghat$ satisfying
\begin{align}
\Rsel(\hat \lambda, \hat \tau) &\le \alpha, \\
\pacc(\hat \lambda, \hat \tau) &\ge \pmin, \\
\Udep(\hat \lambda, \hat \tau) &\ge \Umargin(\alpha, \pmin) - 2 \gammau,
\end{align}
where
\begin{align}
\gammar &:= 4 B \sqrt{\frac{\log(64 m / \delta)}{\ncert \pmin}} + \frac{14 B}{3} \frac{\log(64 m / \delta)}{\pmin (\ncert - 1)}, \\
\gammau &:= \max_{(\lambda, \tau) \in \grid} \etaU(\lambda, \tau).
\end{align}
The factor of two in the margin condition $\pacc \ge 2 \pmin$ is the derived cost of the Clopper--Pearson relative-error inversion (\cref{sublemma:cp-inversion}), not an arbitrary slack. Here $\gammau$ is the realised, \emph{data-dependent} maximum Maurer--Pontil radius computed on $\Dcert$ (it depends on the empirical utility variances $\SigmaU$); the utility guarantee is stated in terms of this realised radius, the quantity \cref{alg:certify} reports.
\end{theorem}

\subsection{Delta ledger}

\Cref{tab:delta-ledger} lists every union-bound contribution to the $(1 - \delta)$ guarantee. The total failure is $\le 3 \delta / 4 \le \delta$, with $\delta / 4$ from \cref{lem:inclusion} and $\delta / 2$ from \cref{lem:utility-closeness}. The H-set restriction at U3 and U4 unions the Chernoff events only over the deterministic eligible set $\Hset := \{(\lambda, \tau) \in \grid : \pacc(\lambda, \tau) \ge \pmin\}$, with $|\Hset| \le m$. This restriction is required for \emph{correctness}, not for budget savings: the U3/U4 Chernoff tail bounds rely on the sample-size hypothesis $\ncert \pacc \ge \ncert \pmin \ge 32 \log(32 m / \delta)$, which holds only for pairs with $\pacc \ge \pmin$; for a pair with $\pacc < \pmin$ the per-event failure guarantees need not hold at all, so such pairs cannot be admitted to the Chernoff union. (The worst-case count is $|\Hset| \le m$ either way, so the formal budget is $\le \delta / 4$ regardless of $|\Hset|$.)

\paragraph{Constant map.} Three constants interact in the proof and deserve named handles. The per-event allocation $\delta / (16 m)$ at U1--U4 is the Bonferroni share for each of the four event families (two-sided MP on $Z$, Clopper--Pearson coverage, Chernoff lower-tail on $s$, Chernoff upper-tail on $s$) at each of $m$ pairs, with the $16$ instead of $4$ absorbing the $4{\times}$ factor that takes the four families to one joint event. The per-event allocation $\delta / (2 m)$ at U6 is the Bonferroni share for the single two-sided MP event on $\udep$ at each of $m$ pairs. The $3 \delta / 4$ headline at U7 comes from $\delta / 4 + \delta / 2$; the slack of $\delta / 4$ is the cost of pessimistic accounting and can be tightened with a more careful joint analysis without changing the asymptotic rate. The $\pmin$ vs $2 \pmin$ split between the certified-floor guarantee ($\pacc \ge \pmin$) and the margin-oracle hypothesis ($\pacc \ge 2 \pmin$) is the derived factor-of-two cost of inverting the Clopper--Pearson lower bound at finite sample size, established in \cref{sublemma:cp-inversion}.

\begin{table}[t]
\centering
\caption{Union-bound delta ledger for \cref{thm:joint-cert}. Subtotals sum to $\le 3 \delta / 4$ at U7; the H-set restriction at U3 and U4 keeps the count at $|\Hset| \le m$ instead of $m$.}
\label{tab:delta-ledger}
\scriptsize
\setlength{\tabcolsep}{3pt}
\begin{tabular}{ccccc}
\toprule
Step & Event family & Per-event level & Count & Subtotal \\
\midrule
U1 & $E_1$: two-sided MP on $Z$            & $\delta / (16 m)$ & $m$               & $\delta / 16$ \\
U2 & $E_2$: Clopper--Pearson coverage     & $\delta / (16 m)$ & $m$               & $\delta / 16$ \\
U3 & $E_3$: Chernoff lower-tail on $s$    & $\delta / (16 m)$ & $|\Hset| \le m$   & $\le \delta / 16$ \\
U4 & $E_4$: Chernoff upper-tail on $s$    & $\delta / (16 m)$ & $|\Hset| \le m$   & $\le \delta / 16$ \\
U5 & \cref{lem:inclusion} joint $E$        & ---                 & ---                 & $\le \delta / 4$ \\
U6 & $F$: two-sided MP on $\udep$        & $\delta / (2 m)$  & $m$               & $\delta / 2$ \\
U7 & \cref{thm:joint-cert} joint $E \cap F$ & ---                 & ---                 & $\le 3 \delta / 4 \le \delta$ \\
\bottomrule
\end{tabular}
\end{table}

\subsection{Clopper--Pearson relative-error inversion}

\begin{lemma}[CP relative-error LCB]\label{sublemma:cp-inversion}
Let $\delta' \in (0,1)$, $\pmin \in (0,1]$, and $s \sim \mathrm{Bin}(\ncert, p)$ with $p \ge \pmin$. Define $\pLCB(s; \ncert, \delta') := \sup\{ q \in [0,1] : \Prob_q[\mathrm{Bin}(\ncert, q) \ge s] \le \delta' \}$, with $\pLCB(0; \ncert, \delta') := 0$. If $\ncert \ge 32 \log(2/\delta') / \pmin$, then with probability at least $1 - \delta'$, $\pLCB(s; \ncert, \delta') \ge p / 2$. Moreover, on the deterministic event $\{s \ge (3/4) \ncert p\}$, the conclusion $\pLCB(s; \ncert, \delta') \ge p/2$ holds without spending additional probability under the same sample-size condition.
\end{lemma}

\begin{proof}
By Chernoff lower-tail with $\varepsilon = 1/4$, $\Prob(s \le \frac{3}{4} \ncert p) \le \exp(-\ncert p / 32) \le \delta'/2$ under the sample-size condition, defining an event $E_s := \{s \ge (3/4) \ncert p\}$. Setting $q := p/2$, the upper tail $\Prob_q[\mathrm{Bin}(\ncert, q) \ge \frac{3}{4} \ncert p] \le \exp(-\ncert q / 10) = \exp(-\ncert p / 20) \le \delta'/2$ follows by Chernoff upper-tail at $\varepsilon = 1/2$. For $s \ge 1$, the Clopper--Pearson characterisation $\pLCB(s; \ncert, \delta') \ge q \Leftrightarrow \Prob_q[\mathrm{Bin}(\ncert, q) \ge s] \le \delta'$ (the binomial upper tail is monotone non-decreasing in $q$, so the CP LCB is the largest $q$ satisfying the inequality; the boundary case $s = 0$, where $\pLCB := 0$, is excluded on $E_s$ because the sample-size condition forces $(3/4) \ncert p > 0$) closes the chain: on $E_s$, $\mathrm{Bin}(\ncert, q) \ge s$ implies $\mathrm{Bin}(\ncert, q) \ge \frac{3}{4} \ncert p$, and the upper-tail bound gives $\Prob_q[\mathrm{Bin}(\ncert, q) \ge s] \le \delta'/2 \le \delta'$. Note that the second clause of the lemma (the deterministic upgrade) uses only the Step-2 Chernoff calculation and the CP monotonicity characterisation, not the Step-1 high-probability event; this is what is invoked in step~(f) of \cref{lem:inclusion}.
\end{proof}

\subsection{Variance-adaptive inclusion lemma (full proof)}

\begin{lemma}[Variance-adaptive inclusion]\label{lem:inclusion}
Under the assumptions of \cref{thm:joint-cert}, conditional on $\Dtr, \Dtune$ (so $g, \grid, \Hset$ and every $\pacc(\lambda, \tau)$ are fixed), with probability at least $1 - \delta / 4$ over the draw of $\Dcert$, for every $(\lambda, \tau) \in \grid$:
\begin{itemize}
\item \emph{Inclusion:} if $\Rsel(\lambda, \tau) \le \alpha - \gammar$ and $\pacc(\lambda, \tau) \ge 2 \pmin$, then $(\lambda, \tau) \in \Ghat$.
\item \emph{Validity:} if $(\lambda, \tau) \in \Ghat$, then $\Rsel(\lambda, \tau) \le \alpha$ and $\pacc(\lambda, \tau) \ge \pmin$.
\end{itemize}
\end{lemma}

\begin{proof}
Define the deterministic eligible set $\Hset := \{(\lambda, \tau) \in \grid : \pacc(\lambda, \tau) \ge \pmin\}$, with $|\Hset| \le m$. We union four event families at per-event level $\delta / (16 m)$, two of which we union over all $m$ pairs and two only over $\Hset$.

\paragraph{Event families.} For each $(\lambda, \tau)$:
\begin{enumerate}
\item[(U1)] $E_1(\lambda, \tau) := \{|\Zbar(\lambda, \tau) - \E[Z(\lambda, \tau)]| \le \etaZ(\lambda, \tau)\}$, with $\etaZ := \sqrt{2 \SigmaZ \log(64 m / \delta) / \ncert} + 7 B \log(64 m / \delta) / (3 (\ncert - 1))$. By two-sided Maurer--Pontil applied as the union of two one-sided bounds at $\delta / (32 m)$ per tail, $\Prob(E_1) \ge 1 - \delta / (16 m)$; the log argument $\log(64 m / \delta)$ tracks this two-tail union. $E_1$ holds for arbitrary $\pacc$ and is unioned over all $m$ pairs.
\item[(U2)] $E_2(\lambda, \tau) := \{\pLCB(s(\lambda, \tau); \ncert, \delta / (16 m)) \le \pacc(\lambda, \tau)\}$. By the definition of the Clopper--Pearson LCB, $\Prob(E_2) \ge 1 - \delta / (16 m)$. $E_2$ holds for arbitrary $\pacc \in [0, 1]$ and is unioned over all $m$ pairs.
\item[(U3)] $E_3(\lambda, \tau) := \{s(\lambda, \tau) \ge \tfrac{3}{4} \ncert \pacc(\lambda, \tau)\}$, $(\lambda, \tau) \in \Hset$. By Step 1 of \cref{sublemma:cp-inversion} applied at level $\delta'' = \delta / (8 m)$ (so the Step-1 failure is $\delta''/2 = \delta / (16 m)$ and the sample-size precondition $\ncert \ge 32 \log(16 m / \delta) / \pmin$ is implied by $(\star)$), $\Prob(E_3) \ge 1 - \delta / (16 m)$ when $\pacc \ge \pmin$. We union only over $\Hset$.
\item[(U4)] $E_4(\lambda, \tau) := \{s(\lambda, \tau) \le \tfrac{3}{2} \ncert \pacc(\lambda, \tau)\}$, $(\lambda, \tau) \in \Hset$. By Chernoff upper-tail at $\varepsilon = 1/2$, $\Prob_p[s \ge \tfrac{3}{2} \ncert p] \le \exp(-\ncert p / 10)$. For $(\lambda, \tau) \in \Hset$ and under $(\star)$, $\ncert p \ge \ncert \pmin \ge 32 \log(32 m / \delta) \ge 10 \log(16 m / \delta)$, so $\exp(-\ncert p / 10) \le \delta / (16 m)$. We union only over $\Hset$.
\end{enumerate}
The joint event $E := \bigcap_{(\lambda, \tau)} (E_1 \cap E_2) \cap \bigcap_{(\lambda, \tau) \in \Hset} (E_3 \cap E_4)$ has failure $\le (2 m + 2 |\Hset|) \delta / (16 m) \le 4 m \cdot \delta / (16 m) = \delta / 4$.

\paragraph{Validity.} Fix $(\lambda, \tau) \in \Ghat$. By the definition of $\Ghat$, $\EB(\lambda, \tau) = \Zbar + \etaZ \le 0$ and $\pLCB(\lambda, \tau) \ge \pmin$. On $E_1$, $\E[Z] \le \Zbar + \etaZ \le 0$, i.e.\ $\E[A (L - \alpha)] \le 0$. On $E_2$, $\pacc \ge \pLCB \ge \pmin > 0$; dividing gives $\Rsel - \alpha \le 0$. Hence $\Rsel \le \alpha$ and $\pacc \ge \pmin$. \checkmark

\paragraph{Inclusion.} Fix $(\lambda, \tau)$ with $\Rsel \le \alpha - \gammar$ and $\pacc \ge 2 \pmin$, so $(\lambda, \tau) \in \Hset$. We show in turn:
\begin{enumerate}
\item[(a)] $\E[Z] = \pacc (\Rsel - \alpha) \le - \pacc \gammar$.
\item[(b)] Since $B \ge 1$ (\cref{ass:bounded}) and $\alpha \in (0,1)$ (\cref{ass:params}), $|L_i - \alpha| \le \max(\alpha, B - \alpha) \le B$, hence $Z_i^2 = A_i (L_i - \alpha)^2 \le B^2 A_i$ and $\sum_i Z_i^2 \le B^2 \ncert \paccHat$. By the standard bound $\SigmaZ = (1 / (\ncert - 1)) (\sum_i Z_i^2 - \ncert \Zbar^2) \le (1 / (\ncert - 1)) \sum_i Z_i^2$ and $\ncert / (\ncert - 1) \le 4 / 3$ under $(\star)$ (which forces $\ncert \ge 32 \log 32 \approx 110$), we obtain $\SigmaZ \le \tfrac{4}{3} B^2 \paccHat$.
\item[(c)] On $E_4$, $\paccHat \le \tfrac{3}{2} \pacc$, so $\SigmaZ \le 2 B^2 \pacc$.
\item[(d)] Substituting into the MP radius, $\etaZ \le 2 B \sqrt{\pacc \log(64 m / \delta) / \ncert} + 7 B \log(64 m / \delta) / (3 (\ncert - 1))$.
\item[(e)] On $E_1$, $\EB = \Zbar + \etaZ \le \E[Z] + 2 \etaZ \le - \pacc \gammar + 2 \etaZ$; since this is an upper bound on $\EB$, a \emph{sufficient} condition for $\EB \le 0$ is $\pacc \gammar \ge 2 \etaZ$. Substituting (d),
\[
\pacc \gammar \;\ge\; 4 B \sqrt{\frac{\pacc \log(64 m / \delta)}{\ncert}} + \frac{14 B \log(64 m / \delta)}{3 (\ncert - 1)}.
\]
Dividing by $\pacc$ and using $\pacc \ge \pmin$ to bound $1/\sqrt{\pacc} \le 1/\sqrt{\pmin}$ and $1/\pacc \le 1/\pmin$, the right-hand side is at most the displayed expression for $\gammar$, so $\EB \le 0$ holds whenever
\[
\gammar \ge \frac{4 B}{\sqrt{\pmin}} \sqrt{\frac{\log(64 m / \delta)}{\ncert}} + \frac{14 B}{3 \pmin} \cdot \frac{\log(64 m / \delta)}{\ncert - 1},
\]
which is exactly the definition of $\gammar$ given in \cref{thm:joint-cert}. Hence $\EB \le 0$.
\item[(f)] On $E_3$, $s \ge (3/4) \ncert \pacc$. Set $\delta' := \delta / (16 m)$ and $q := \pacc / 2$. The sample-size precondition $\ncert \ge 32 \log(2/\delta') / \pmin = 32 \log(32 m / \delta) / \pmin$ is exactly $(\star)$, so the Step-2 Chernoff calculation of \cref{sublemma:cp-inversion} gives $\Prob_q[\mathrm{Bin}(\ncert, q) \ge (3/4) \ncert \pacc] \le \delta'$, hence $\Prob_q[\mathrm{Bin}(\ncert, q) \ge s] \le \delta'$. By the CP characterisation, $\pLCB(s; \ncert, \delta') \ge q = \pacc / 2$ deterministically on $E_3$ (no additional probability is spent beyond $E_3$). The inclusion hypothesis $\pacc \ge 2 \pmin$ then gives $\pLCB \ge \pacc / 2 \ge \pmin$, so $(\lambda, \tau) \in \Ghat$. \checkmark
\end{enumerate}
Both directions hold on $E$, which has probability $\ge 1 - \delta / 4$.
\end{proof}

\begin{remark}
The inclusion direction uses $\pacc \ge \pmin$ for the $\gammar$ bound but $\pacc \ge 2 \pmin$ for step (f). The factor of two in the margin condition is therefore the derived cost of the Clopper--Pearson inversion.
\end{remark}

\subsection{Utility closeness lemma (full proof)}

\begin{lemma}[Two-sided closeness on utility]\label{lem:utility-closeness}
Suppose \cref{ass:iid,ass:bounded,ass:grid,ass:params,ass:threesplit,ass:selector} hold and $\ncert \ge 2$ (so the Bessel-corrected $\SigmaU$ and the Maurer--Pontil radius $\etaU$ are well-defined; this is implied by \cref{ass:samplesize} whenever the latter is assumed). Conditional on the training and tuning splits, with probability at least $1 - \delta / 2$ over $\Dcert$, for every $(\lambda, \tau) \in \grid$,
\begin{equation}
\Udep(\lambda, \tau) - 2 \etaU(\lambda, \tau) \;\le\; \ULCB(\lambda, \tau) \;\le\; \Udep(\lambda, \tau).
\label{eq:util-closeness}
\end{equation}
In particular, $|\ULCB(\lambda, \tau) - \Udep(\lambda, \tau)| \le 2 \etaU(\lambda, \tau) \le 2 \gammau$ uniformly over the grid.
\end{lemma}

\begin{proof}
For each $(\lambda, \tau) \in \grid$, define the two-sided Maurer--Pontil event
$
F(\lambda, \tau) := \{ |\ubar(\lambda, \tau) - \Udep(\lambda, \tau)| \le \etaU(\lambda, \tau) \}.
$
The per-sample $\udep_i = A_i v_i - c (1 - A_i)$ lies in $[-c, V]$, so its range is $c + V$. The two-sided MP bound at total level $\delta / (2 m)$, applied as the union of two one-sided MP bounds at $\delta / (4 m)$ per tail, gives $\Prob(F(\lambda, \tau)) \ge 1 - \delta / (2 m)$ with the log argument $\log(8 m / \delta)$ that appears in $\etaU$. Unioning over the $m$ grid pairs, the joint event $F := \bigcap_{(\lambda, \tau)} F(\lambda, \tau)$ has failure $\le m \cdot \delta / (2 m) = \delta / 2$. On $F$, $\bar u - \etaU \le \Udep \le \bar u + \etaU$ for every pair; substituting $\ULCB := \bar u - \etaU$ gives both inequalities in (\ref{eq:util-closeness}): the upper $\ULCB \le \Udep$ from $\bar u \le \Udep + \etaU$, and the lower $\ULCB \ge \Udep - 2 \etaU$ from $\bar u \ge \Udep - \etaU$.
\end{proof}

\subsection{Proof of Theorem 1}

\begin{proof}[Proof of \cref{thm:joint-cert}]
Let $E$ be the event of \cref{lem:inclusion} (probability $\ge 1 - \delta / 4$) and $F$ the event of \cref{lem:utility-closeness} (probability $\ge 1 - \delta / 2$). By a union bound, $\Prob(E \cap F) \ge 1 - (\delta / 4 + \delta / 2) = 1 - 3 \delta / 4 \ge 1 - \delta$. We work on $E \cap F$.

\paragraph{Case A: $\Ghat = \emptyset$.} \cref{alg:certify} returns \textsc{Infeasible}; the theorem statement is satisfied vacuously by the disjunction.

\paragraph{Case B: $\Ghat \ne \emptyset$ and $(\hat \lambda, \hat \tau) = \argmax_{(\lambda, \tau) \in \Ghat} \ULCB(\lambda, \tau)$.}

\emph{Validity and floor.} Since $(\hat \lambda, \hat \tau) \in \Ghat$, the validity direction of \cref{lem:inclusion} gives $\Rsel(\hat \lambda, \hat \tau) \le \alpha$ and $\pacc(\hat \lambda, \hat \tau) \ge \pmin$.

\emph{Margin-oracle utility.} If $\Mset(\alpha - \gammar, \pmin) = \emptyset$, then $\Umargin(\alpha, \pmin) = \sup_\emptyset = -\infty$ and the inequality $\Udep(\hat \lambda, \hat \tau) \ge -\infty$ holds trivially. Otherwise, fix any $(\lambda', \tau') \in \Mset(\alpha - \gammar, \pmin)$; by the inclusion direction of \cref{lem:inclusion}, $(\lambda', \tau') \in \Ghat$. Chain three inequalities:
\begin{align}
\Udep(\hat \lambda, \hat \tau) &\;\ge\; \ULCB(\hat \lambda, \hat \tau) &&\text{(\cref{lem:utility-closeness}, upper)} \label{eq:thm-T3} \\
                                &\;\ge\; \ULCB(\lambda', \tau')      &&\text{($\argmax$ at $(\hat \lambda, \hat \tau)$)} \label{eq:thm-T1} \\
                                &\;\ge\; \Udep(\lambda', \tau') - 2 \gammau &&\text{(\cref{lem:utility-closeness}, lower)} \label{eq:thm-T2}
\end{align}
Combining, $\Udep(\hat \lambda, \hat \tau) \ge \Udep(\lambda', \tau') - 2 \gammau$. Taking the supremum over $(\lambda', \tau') \in \Mset(\alpha - \gammar, \pmin)$ gives $\Udep(\hat \lambda, \hat \tau) \ge \Umargin(\alpha, \pmin) - 2 \gammau$.
\end{proof}

\subsection{The utility leg in three rungs}
\label{sec:utility-ladder}

\Cref{thm:joint-cert} states the utility guarantee in its most ambitious, external-oracle form (\ref{eq:thm-T2}). We separate, on the same event $E \cap F$, three nested utility guarantees that the certified pair enjoys, from the always-available to the most ambitious. The first two hold on \emph{every} feasible run; only the third can be vacuous, and exactly when its oracle set is empty.

\begin{enumerate}
\item \emph{Absolute certified utility.} The upper direction of \cref{lem:utility-closeness} gives $\Udep(\hat\lambda, \hat\tau) \ge \ULCB(\hat\lambda, \hat\tau)$: \cref{alg:certify} reports a finite-sample lower bound on the deployed policy's utility, non-vacuous whenever the certifier is feasible.
\item \emph{Certified-set optimality (\cref{cor:gset-opt}).} The deployed pair is within $2\gammau$ of the best deployment utility \emph{among all certified pairs} $\Ghat$, a comparator that is non-empty whenever \cref{alg:certify} does not abstain.
\item \emph{External margin-oracle optimality (\cref{thm:joint-cert}).} The deployed pair is within $2\gammau$ of the best utility over an \emph{external}, certificate-independent oracle of risk-certifiable policies. This is the strongest statement, but is informative only when the oracle set $\Mset(\alpha - \gammar, \pmin)$ is non-empty, i.e.\ when $\gammar < \alpha$; \cref{cor:va-oracle} enlarges this set variance-adaptively.
\end{enumerate}

\begin{corollary}[Certified-set utility optimality]\label{cor:gset-opt}
Under the assumptions of \cref{thm:joint-cert}, on the event $E \cap F$ of probability at least $1 - \delta$, if $\Ghat \ne \emptyset$ the returned pair satisfies both
\[
\Udep(\hat\lambda, \hat\tau) \;\ge\; \ULCB(\hat\lambda, \hat\tau)
\qquad\text{and}\qquad
\Udep(\hat\lambda, \hat\tau) \;\ge\; \sup_{(\lambda, \tau) \in \Ghat} \Udep(\lambda, \tau) \;-\; 2 \gammau .
\]
Unlike $\Umargin$, the comparator $\Ghat$ is non-empty on every feasible run, so both bounds are non-vacuous whenever \cref{alg:certify} does not return \textsc{Infeasible}.
\end{corollary}

\begin{proof}
On $F$ (\cref{lem:utility-closeness}, upper direction), $\Udep(\hat\lambda, \hat\tau) \ge \ULCB(\hat\lambda, \hat\tau)$, the first inequality. For the second, fix any $(\lambda, \tau) \in \Ghat$. Since $(\hat\lambda, \hat\tau) = \argmax_{(\lambda,\tau) \in \Ghat} \ULCB$, we have $\ULCB(\hat\lambda, \hat\tau) \ge \ULCB(\lambda, \tau)$; and by the lower direction of \cref{lem:utility-closeness}, $\ULCB(\lambda, \tau) \ge \Udep(\lambda, \tau) - 2 \etaU(\lambda, \tau) \ge \Udep(\lambda, \tau) - 2 \gammau$. Chaining the three inequalities and taking the supremum over $(\lambda, \tau) \in \Ghat$ gives the claim. The comparator differs from $\Umargin$ only in ranging over the realised certified set rather than an external population set; the same $\argmax$-and-closeness argument underlies both.
\end{proof}

\begin{corollary}[Variance-adaptive margin oracle]\label{cor:va-oracle}
For each pair let $\sigma_Z^2(\lambda, \tau) := \mathrm{Var}(Z(\lambda, \tau))$ be the population variance of $Z = A(L - \alpha)$, and define the variance-adaptive margin
\[
\gammar^{\mathrm{va}}(\lambda, \tau) \;:=\; \frac{2}{\pacc(\lambda, \tau)} \left[ \sqrt{\frac{2 \sigma_Z^2(\lambda, \tau) \log(64 m / \delta)}{\ncert}} + \frac{c_0 B \log(64 m / \delta)}{\ncert - 1} \right]
\]
for an absolute constant $c_0$, and the variance-adaptive oracle
$\Umarginva := \sup \{ \Udep(\lambda, \tau) : (\lambda, \tau) \in \grid \ \wedge\ \Rsel(\lambda, \tau) + \gammar^{\mathrm{va}}(\lambda, \tau) \le \alpha \ \wedge\ \pacc(\lambda, \tau) \ge 2 \pmin \}$ (with $\sup_\emptyset := -\infty$; the supremum is over grid pairs, matching $\Mset$ and the grid-only reach of the inclusion lemma). Then \cref{thm:joint-cert} holds verbatim with $\Umargin$ replaced by $\Umarginva$. Because $\sigma_Z^2(\lambda, \tau) \le B^2 \pacc(\lambda, \tau)$, on the oracle domain $\pacc \ge 2 \pmin$ the leading term of $\gammar^{\mathrm{va}}$ never exceeds that of the worst-case $\gammar$ and is strictly smaller on low-variance pairs ($\sigma_Z^2 \ll B^2 \pacc$); consequently $\Umarginva$ can be finite (the oracle non-empty) on surfaces where the worst-case oracle $\Mset(\alpha - \gammar, \pmin)$ is empty. The oracle is defined through the population $\Rsel, \pacc, \sigma_Z^2$ and is not computable from a single calibration draw; \cref{sec:experiments} reports a calibration plug-in surrogate as an empirical diagnostic of non-emptiness, not a population-level certificate. The proof (supplementary material, Appendix~A) reruns the inclusion lemma with the per-pair empirical-Bernstein radius in place of the worst-case variance bound, swapping the Chernoff upper-tail event for one Maurer--Pontil variance-concentration event at the same per-event level, so the $\delta/4$ inclusion ledger is unchanged.
\end{corollary}

\subsection{Lower bound}

\begin{proposition}[Range-only Hoeffding-ratio limitation, sketch]\label{prop:lower-bound}
Fix $\delta \in (0, 1/2)$, $\pmin \in (0, 1]$, $\alpha \in [0, B]$, and suppose $\ncert \pmin \ge C \log(1/\delta)$ for a sufficiently large absolute constant $C$ (a regime condition included only to situate the witness within \cref{thm:joint-cert}'s operating range; the lower bound below uses solely the validity of $\underline{p}$, not this condition). Consider any range-only Hoeffding-ratio construction $\hat R := \alpha + (\Zbar + r_H) / \underline{p}$, where $r_H := \sqrt{B^2 \log(2/\delta)/(2\ncert)}$ is the range-$B$ Hoeffding radius on $\Zbar = \tfrac{1}{\ncert}\sum_i A_i(L_i - \alpha)$ and $\underline{p} \in (0,1]$ is any $(1-\delta)$-valid lower confidence bound on $\E[A]$ (i.e.\ $\Prob[\underline{p} \le \E[A]] \ge 1-\delta$). Then there exists a joint distribution on $(A, L)$ with $A \in \{0, 1\}$, $L \in [0, B]$, $\E[A] = \pmin$ \emph{exactly}, such that, with probability at least $1 - \delta$ under that distribution, the numerator-radius contribution $r_H / \underline{p}$ to $\hat R - \alpha$ is at least $\Omega(B \sqrt{\log(1/\delta)} / (\pmin \sqrt{\ncert}))$. This is a limitation of the stated range-only construction, not a minimax lower bound for all selective conformal risk control procedures.
\end{proposition}

The witness takes $A \sim \mathrm{Bernoulli}(\pmin)$ (so $\E[A] = \pmin$ exactly) and, conditional on $A = 1$, $L$ supported on $\{0, B\}$ with nonzero mass on both endpoints (so the nonzero support of $Z = A(L - \alpha)$ spans $B$); when $A = 0$, $L$ is arbitrary. The numerator radius $r_H = \Theta(B \sqrt{\log(1/\delta)/\ncert})$ is fixed by the construction. Because $\E[A] = \pmin$ exactly, validity of the LCB gives $\Prob[\underline{p} \le \pmin] = \Prob[\underline{p} \le \E[A]] \ge 1 - \delta$; on that event $r_H / \underline{p} \ge r_H / \pmin = \Omega(B \sqrt{\log(1/\delta)} / (\pmin \sqrt{\ncert}))$. (The lower bound needs an \emph{upper} bound on the denominator $\underline{p}$, supplied by the validity direction at the exact-$\pmin$ witness, not the concentration direction $\underline{p} \ge \pmin/2$.) The full argument, with the construction class and probability mode made precise, is in the supplementary material (Appendix~C). The variance-adaptive rate $\gammar = O(B \sqrt{\log(64m/\delta) / (\ncert \pmin)})$ in \cref{thm:joint-cert} (under $(\star)$) shares the $B$ and $\ncert^{-1/2}$ scaling with the range-only construction and improves the acceptance dependence from $1/\pmin$ to $1/\sqrt{\pmin}$ (up to the $\sqrt{\log(64m/\delta)/\log(1/\delta)}$ Bonferroni overhead). We read this as a structural variance-adaptive advantage at low acceptance against the stated range-only Hoeffding-ratio construction, not as a minimax-over-class optimality claim. \cref{prop:lower-bound} is a sketch over a single two-point Bernoulli distribution and a limitation result for the specific range-only Hoeffding-ratio construction, not a full minimax lower bound over a class of joint distributions and procedures; the latter, including a minimax lower bound for the adaptive selective conformal risk control class, is open.

\subsection{Regime separation between Ours and Hoeffding--CRC}

A natural question is when our variance-adaptive per-pair risk-bound expression is numerically tighter than the Hoeffding--CRC selective per-pair expression on the same pair. The two per-pair expressions, written at matched Bonferroni regimes, are
\begin{align}
\UCBours(\lambda, \tau) &= \RselHat + \etaZ(\lambda, \tau) / \paccHat, \\
\UCBHoeff(\lambda, \tau) &= \RselHat + B \sqrt{\log(m / \delta) / (2 s)},
\end{align}
where $s = \ncert \paccHat$ is the observed acceptance count and $\etaZ$ uses the empirical variance $\SigmaZ$ at the conservative level $\delta / (16 m)$. We compare these two \emph{nominal} per-pair half-widths. The numerator $\Zbar + \etaZ$ of $\UCBours$ is the certified-valid quantity tested by $\EB \le 0$ in \cref{lem:inclusion}; the division by the empirical $\paccHat$ places it on the $\Rsel$ scale for comparison with the Hoeffding form, and is the per-pair point-estimate denominator used in the selective-CRC literature for both constructions. A standalone $(1-\delta)$-valid per-pair upper bound on $\Rsel$ would instead divide by the acceptance LCB, giving $\alpha + (\Zbar + \etaZ)/\pLCB$; \cref{cor:regime} compares the nominal expressions, and the certificate's validity (\cref{thm:joint-cert}) is established through the $\EB \le 0$ test, not through $\UCBours$ being a standalone valid bound on $\Rsel$.

\begin{corollary}[Ours-vs-\HCRC{} regime separation, finite-sample and leading-order]\label{cor:regime}
Fix a pair with $s = \ncert \paccHat \ge 2$ and let $\kappa_n := \ncert / (\ncert - 1)$. Let $\LO := \log(64 m / \delta)$ and $\LH := \log(m / \delta)$. The exact algebraic identity (derived in the supplementary, Appendix~A)
\[
\SigmaZ = \frac{s-1}{\ncert - 1} \sigmaHat + \kappa_n \paccHat (1 - \paccHat) (\RselHat - \alpha)^2
\]
gives $\Tex := \SigmaZ / \paccHat = \kappa_n [(1 - 1/s) \sigmaHat + (1 - \paccHat)(\RselHat - \alpha)^2]$. Then the finite-sample per-pair comparison $\UCBours < \UCBHoeff$ is equivalent to $\Tex < \sigmaexact(s)$, where
\begin{equation}
\sigmaexact(s) := \frac{B^2}{2 \LO} \left[ \sqrt{\LH / 2} - \kappa_n \frac{7 \LO}{3 \sqrt s} \right]_+^2.
\label{eq:cor-V-exact}
\end{equation}
Replacing $\Tex$ by the leading-order surrogate
$\Tobs := \sigmaHat + (1 - \paccHat)(\RselHat - \alpha)^2$
and $\kappa_n$ by $1$ gives the leading-order predictor
\begin{equation}
\Tobs < \sigmastar(s) := \frac{B^2}{2 \LO} \left[ \sqrt{\LH / 2} - \frac{7 \LO}{3 \sqrt s} \right]_+^2,
\label{eq:cor-V}
\end{equation}
where $[x]_+ := \max(x, 0)$. Both the variable side and the threshold side incur Bessel corrections, with \emph{explicit absolute constants} (proved in the supplementary material, Appendix~A): for every $\ncert \ge 2$ and $s \ge 2$,
\begin{equation}
\begin{aligned}
|\Tex - \Tobs| &\le 3 B^2 (1/s + 1/\ncert), \\
|\sigmaexact(s) - \sigmastar(s)| &\le B^2 / \ncert,
\end{aligned}
\label{eq:cor-explicit-slack}
\end{equation}
both independent of $m, \delta, s, \ncert$. A sufficient (and now checkable) condition for the leading-order rule (\ref{eq:cor-V}) to agree with the finite-sample iff (\ref{eq:cor-V-exact}) on a given pair is therefore
\[
|\Tobs - \sigmastar(s)| > 4 B^2 (1/s + 1/\ncert);
\]
pairs within this explicit band near the threshold may be classified differently by the two rules. When $\paccHat$ is bounded below (so $1/s = O(1/\ncert)$), the combined slack reduces to the standard $O(B^2/\ncert)$ leading-order error. Condition (\ref{eq:cor-V}) is satisfiable ($\sigmastar(s) > 0$) only when $s > \sznot := [14 \LO / (3 \sqrt{2 \LH})]^2$.
\end{corollary}

\begin{proof}[Proof sketch]
From the exact identity, $\etaZ \le \sqrt{2 \paccHat \Tex \LO / \ncert} + 7 B \LO / (3 (\ncert - 1))$. Dividing by $\paccHat = s / \ncert$ and multiplying by $\sqrt s$ gives $\sqrt s \cdot \etaZ / \paccHat \le \sqrt{2 \Tex \LO} + \kappa_n \cdot 7 B \LO / (3 \sqrt s)$. The Hoeffding side is exact: $\sqrt s \cdot (\UCBHoeff - \RselHat) = B \sqrt{\LH / 2}$. Rearranging $\sqrt{2 \Tex \LO} + \kappa_n \cdot 7 B \LO / (3 \sqrt s) < B \sqrt{\LH / 2}$ for $\Tex$ gives (\ref{eq:cor-V-exact}). Replacing $\Tex$ by $\Tobs$ and $\kappa_n$ by $1$ yields (\ref{eq:cor-V}); the bounded-loss bound $|\Tex - \Tobs| \le \kappa_n \sigmaHat / s + (\kappa_n - 1) [\sigmaHat + (1 - \paccHat)(\RselHat - \alpha)^2] \le 3 B^2 (1/s + 1/\ncert)$ for $\ncert \ge 2$, together with the threshold-side gap $|\sigmaexact(s) - \sigmastar(s)| \le B^2/\ncert$ (both with explicit absolute constants; full derivation including the positive-part Lipschitz step in the supplementary material, Appendix~A), controls the error in (\ref{eq:cor-explicit-slack}). The sample-size readout follows from $\sigmastar(s) > 0$ requiring $7 \LO / (3 \sqrt s) < \sqrt{\LH / 2}$.
\end{proof}

\paragraph{Numerical instantiation.} At $m = 15$, $\delta = 0.10$, $B = 1$ (the segmentation operating point of \cref{sec:experiments}), the threshold sample size is $\sznot \approx 183$. \Cref{tab:cor-numeric} instantiates five representative pairs covering COCO val 2017, ADE20K, and an ImageNet low-acceptance regime: the closed-form prediction $\Tobs < \sigmastar(s)$ matches the observed per-pair outcome in every case. The five pairs are a representative instantiation; the systematic per-(grid, seed) audit across all grid pairs and seeds ($3{,}750$ cells on COCO, spanning three loss families and two calibration scales, and ImageNet) is reported in the supplementary material (Appendix~B), where the closed-form rule matches the realised per-pair winner on every cell, with disagreements provably confined to (and empirically absent even within) the near-threshold band of (\ref{eq:cor-explicit-slack}).

\begin{table*}[t]
\centering
\caption{Numerical instantiation of \cref{cor:regime} (V) on five representative pairs, $B = 1$ (the four segmentation rows use $m = 15$, $\delta = 0.10$; the ImageNet RN50 V2 row uses the ImageNet operating point $m = 35$, $\delta = 0.05$). The closed-form prediction $\Tobs < \sigmastar(s)$ matches the observed per-pair outcome in every row.}
\label{tab:cor-numeric}
\footnotesize
\setlength{\tabcolsep}{4pt}
\begin{tabular}{lrrrrcc}
\toprule
Setup & $s$ & $\sigmastar(s)$ & $\sigmaHat$ & $\Tobs$ & Prediction & Outcome \\
\midrule
COCO pixacc, Ours pair, $\hat p {=} 0.22$, $\ncert {=} 4000$       & 880  & $0.041$ & $0.006$ & $0.007$ & Ours      & $+22.1$ pp \checkmark \\
COCO pixacc, $q {=} 0.85$, $\hat p {=} 0.15$, $\ncert {=} 4000$    & 600  & $0.027$ & $0.005$ & $0.007$ & Ours      & Ours wins \checkmark \\
ADE20K binary, $q {=} 0.50$, $\hat p {=} 0.50$, $\ncert {=} 1500$ & 750  & $0.035$ & $0.087$ & $0.092$ & \HCRC{}   & $-4.3$ pp \checkmark \\
ADE20K binary, $q {=} 0.85$, $\hat p {=} 0.15$, $\ncert {=} 1500$ & 225  & $0.001$ & $0.039$ & $0.061$ & \HCRC{}   & \HCRC{} wins \checkmark \\
ImageNet RN50 V2 low-acc, $\hat p {=} 0.01$, $\ncert {=} 33000$    & 330  & $0.009$ & $0.005$ & $0.005$ & Ours      & $8.4{\times}$ \checkmark \\
\bottomrule
\end{tabular}
\end{table*}

\paragraph{Scope.} \Cref{cor:regime} compares our per-pair upper bound on $\Rsel$ against the Hoeffding--CRC selective per-pair upper bound on the same pair only. It does not address per-pair Bernstein-on-accepted-samples comparisons (a tighter per-pair primitive that does not certify the joint object; \cref{sec:discussion}), per-pair WSR comparisons (regime-dependent), or the certified-acceptance gap (\cref{sec:experiments}), which depends on the certifier and not on a single per-pair upper bound. The corollary's main use is converting the \cref{sec:experiments} regime characterisation from post-hoc empirical contrast into theory-derived regime prediction within the Ours-vs-\HCRC{} comparison scope.

\paragraph{Width scaling.} The three closed-form scaling laws established above (the acceptance-floor rate of \cref{thm:joint-cert} against the range-only construction of \cref{prop:lower-bound}, the per-pair sample-size decay, and the variance adaptivity of \cref{cor:regime}) are visualised together in \cref{fig:width-scaling}. All curves are the proof expressions evaluated at the paper's operating points; no calibration data enters.

\begin{figure*}[t]
\centering
\includegraphics[width=\textwidth]{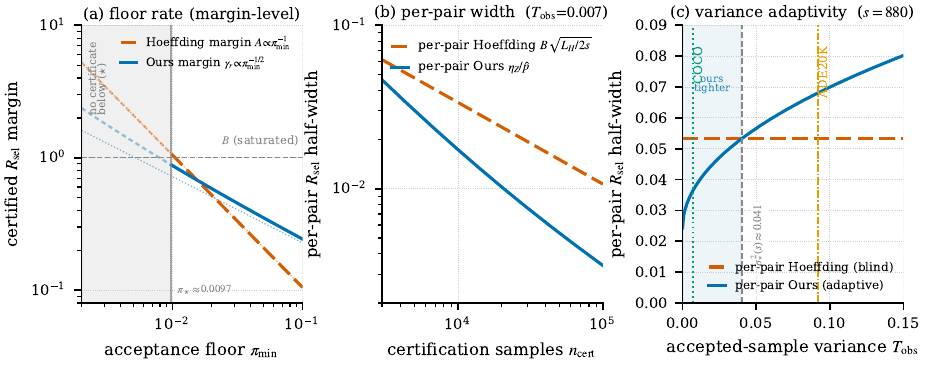}
\caption{Closed-form width scaling of the joint certificate (proof expressions only; no calibration data). Panel (a) is at the \emph{margin} level; panels (b) and (c) are \emph{per-pair} half-widths and are not on the same target or vertical scale as (a). The two ``Hoeffding'' baselines also differ: panel (a) uses the textbook $\pmin$-saturated selective margin $A(\pmin)$, panels (b) and (c) the per-pair upper-bound half-width $B\sqrt{\log(m/\delta)/(2s)}$.
\textbf{(a) Floor rate (margin-level).} In leading order the certified $\Rsel$-margin $\gammar$ (\cref{thm:joint-cert}) scales as $\pmin^{-1/2}$, whereas the textbook $\pmin$-saturated Hoeffding--CRC selective margin $A(\pmin) = (B / \pmin) \sqrt{\log(2 m / \delta) / (2 \ncert)}$ scales as $\pmin^{-1}$ (dotted slope guides). At the ImageNet operating point ($\ncert {=} 33$k, $m {=} 35$, $\delta {=} 0.05$) the Hoeffding margin saturates at the loss range $B$ and certifies nothing across the low-acceptance regime in which $\gammar$ remains below $B$; the two curves cross near $\pmin {\approx} 0.015$, so the advantage is regime-scoped to low acceptance rather than uniform. The faded/dashed segments left of $\pi_\star {\approx} 0.0097$ mark $\pmin$ below the sample-size condition $(\star)$, where no certificate is asserted. This is the leading-order $1 / \pmin \to 1 / \sqrt{\pmin}$ improvement of \cref{thm:joint-cert} over the range-only Hoeffding-ratio construction of \cref{prop:lower-bound}.
\textbf{(b) Per-pair width vs.\ sample size.} At a fixed small accepted-sample variance $\Tobs {=} 0.007$ (the COCO certifier-selected pair; $m {=} 15$, $\delta {=} 0.10$, $\paccHat {=} 0.22$, $s {=} \paccHat \ncert$), the variance-adaptive per-pair half-width $\etaZ / \paccHat$ is tighter than the range-only per-pair Hoeffding half-width $B \sqrt{\log(m / \delta) / (2 s)}$ throughout the plotted range: about $1.3 {\times}$ at $\ncert {=} 3$k rising to $3.1 {\times}$ at $\ncert {=} 100$k, approaching the $\approx 4.4 {\times}$ asymptotic advantage; the Hoeffding half-width is exactly $\propto \ncert^{-1/2}$ while Ours decays at least as fast (visibly steeper at small $\ncert$ owing to its lower-order term).
\textbf{(c) Per-pair variance adaptivity.} Sweeping the accepted-sample variance $\Tobs$ at the segmentation operating point ($s {=} 880$), $\etaZ / \paccHat$ grows as $\sqrt{\Tobs}$ while the per-pair Hoeffding half-width is variance-blind (flat); the two cross at the regime threshold $\sigmastar(880) {\approx} 0.041$ of \cref{cor:regime}~(\ref{eq:cor-V}), which coincides visually with the crossing. The observed COCO variance ($\Tobs {=} 0.007$) lies in the Ours-tighter region; ADE20K ($\Tobs {=} 0.092$) is shown against the same $s {=} 880$ threshold for scale and falls on the Hoeffding side: under its own smaller $s$ its threshold is even lower ($\sigmastar {=} 0.035$), so the verdict is unchanged, matching \cref{tab:cor-numeric}.}
\label{fig:width-scaling}
\end{figure*}

\section{Experimental Evaluation}
\label{sec:experiments}

\paragraph{Setup.} We evaluate on six surfaces. ImageNet val ($50{,}000$ images, ResNet-50/101/152 V2 logits, top-1 of $0.8084$/$0.8191$/$0.8228$) drives the classification headline; CIFAR-100 ($10{,}000$ test, ResNet-56) and ImageNet-V2 matched-frequency \cite{recht2019_imagenetv2} test distribution-shift behaviour; a synthetic Beta(2, 5) loss is a controlled sanity check. COCO val 2017 ($5000$ images, Mask2Former-Swin-B \cite{cheng2022_mask2former} pretrained on COCO panoptic, 133 classes) and ADE20K ($2000$ images, Mask2Former-Swin-B and SegFormer-MiT-B2 \cite{xie2021_segformer} pretrained on ADE20K) drive the dense-pattern-analysis evaluation. The certificate-backed headline results are on the $(\star)$-compliant ImageNet and COCO surfaces; ADE20K and ImageNet-V2 do \emph{not} meet $(\star)$ at their configured parameters and are reported only as out-of-scope stress tests (\cref{sec:oos-stress}), read as heuristic and regime-contrast behaviour rather than certificate-validity evidence. Default operating point is $\alpha = 0.05$, $\delta = 0.05$, $m = 35$ for ImageNet and $\alpha = 0.10$, $\delta = 0.10$, $m = 15$ for segmentation. We use $30$ seeds for the headline ImageNet width comparison, the F.1 stress sweep, and the headline COCO softmax run; $20$ for cross-model and ADE20K evaluations (the COCO entropy variant also uses $20$); and $10$ for the ImageNet-V2 distribution-shift block. All experiments are deterministic given seed. The implementation, the precomputed result files backing every table and figure, the cached COCO per-image arrays, and a self-contained utility-leg verification script are included with this submission (\texttt{code/}; see \texttt{code/README.md}), and mirrored at \url{https://github.com/yahiko-l/joint-selective-crc}.

\subsection{Joint-certificate validity}
\label{sec:validity}

\begin{table*}[t]
    \centering
    \caption{Joint-certificate held-out validity across four evaluation surfaces.
    ``\#feas'' is the number of feasible runs (Algorithm~1 did not return INFEASIBLE);
    ``\#vio'' counts joint violations ($R_{\mathrm{test}} > \alpha$ or $p_{\mathrm{test}} < \pi_{\min}$);
    ``CP 95\% UB'' is the Clopper--Pearson 95\% one-sided upper bound on the violation rate
    (avoids the invalid $[0/n,0/n]$ artefact at zero observed events).
    $^\dagger$B11 is reported descriptively: the sample-size condition (\ref{eq:sample-size}) is not met at the configured $(\pi_{\min}, m)$;
    we make no theorem-backed shift-robustness claim.}
    \label{tab:validity}
    \small
    \begin{tabular}{lrrl}
        \toprule
        Surface & \#feas & \#vio & CP 95\% UB \\
        \midrule
        Synthetic Beta$(2,5)$, $27$ configs $\times$ $20$ seeds & 237 / 540 & 0 & 0.0126 \\
        Synthetic Beta$(2,5)$ stress sweep, $16$ configs $\times$ $30$ seeds & 150 & 0 & 0.0198 \\
        ImageNet RN50 V2, symmetric label-flip $\{0,5,10,20\}\%$ & 40 (4 rates) & 0 & 0.0722 \\
        ImageNet$\to$ImageNet-V2 shift, $\Delta$top-1 $=10.9$pp & 10 & 0 & 0.2589 $^\dagger$ \\
        \bottomrule
    \end{tabular}
\end{table*}

\Cref{tab:validity} reports the empirical violation rate on four surfaces. On the F.1 synthetic Beta(2, 5) stress sweep ($30$ seeds $\times$ $16$ configurations producing $150$ feasible runs in total: of the $16$ configurations, $4$ at $\ncert = 25{,}000$ are fully feasible ($30/30$ each), $4$ at $\ncert = 5000$ are partially feasible ($7$--$8/30$ each), and $8$ are infeasible at the more aggressive settings), the observed violation rate is $0$ with a Clopper--Pearson $95\%$ one-sided upper bound of $0.0198$, well below $\delta = 0.10$. On the synthetic Beta(2, 5) calibration sanity ($237$ of $540$ feasible), the violation rate is $0$. The label-noise robustness block (B12, $40$ runs across noise rates $\{0, 5\%, 10\%, 20\%\}$) records zero violations at every noise level. The ImageNet-V2 distribution-shift block (B11, $10$ runs) records zero violations descriptively: this row does not meet $(\star)$ at the configured $\pmin = 0.01$, $m = 35$ (the configured $\ncert = 25{,}000$ is below the required $\approx 32{,}000$), so the descriptive label is the honest one. Empirically observed zero violations are supporting evidence, not a guarantee; \cref{thm:joint-cert} is the actual guarantee. \Cref{fig:validity-dist} shows the per-seed realised test-risk distribution behind these rates across six surfaces (including the COCO and ADE20K segmentation surfaces of \cref{sec:segmentation}): every empirical $\Rsel^{\mathrm{test}}$ is at or below $\alpha$ except a single ADE20K--SegFormer test exceedance ($\Rsel^{\mathrm{test}} > \alpha$; $1/20$, $+0.25$ pp). Such an exceedance does not by itself contradict \cref{thm:joint-cert}, which controls the population selected risk over calibration draws; the guarantee is the theorem, not the empirical rate (per-row counts and Clopper--Pearson bounds in \cref{tab:validity}).

\begin{figure}[H]
\centering
\includegraphics[width=0.65\textwidth]{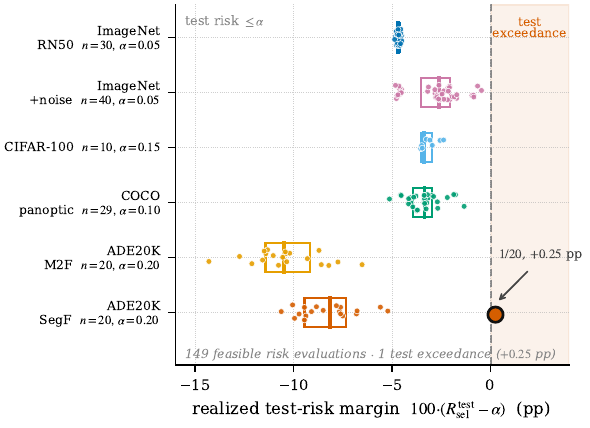}
\caption{Realised test-risk margins across six surfaces, an empirical risk-side sanity check complementing \cref{tab:validity}. Each dot is one calibration/test split (COCO: $29$ of $30$ attempted seeds feasible); boxes mark median and IQR, and per-row $n$ is shown in the axis labels. The plotted quantity is the \emph{empirical} test risk $\Rsel^{\mathrm{test}}$, an estimate of the population selected risk $\Rsel$ that \cref{thm:joint-cert} controls, re-centred as $100 \cdot (\Rsel^{\mathrm{test}} - \alpha)$ so that surfaces with different targets share one threshold. Every $\Rsel^{\mathrm{test}}$ is at or below $\alpha$ except one (ADE20K--SegFormer, $1/20$, $+0.25$ pp), an empirical \emph{test exceedance} (\cref{tab:validity} tabulates these as conservative held-out violation rates). Such an exceedance does not by itself contradict \cref{thm:joint-cert}: it is compatible with finite-test noise or with an allowed $\delta$-level calibration draw, and the guarantee is the theorem, not the empirical rate. Read each row against the zero line, \emph{not} across surfaces as a quality ranking, since the absolute margin reflects $\alpha$ and task difficulty rather than certificate tightness. Per-surface seed counts ($10$--$40$) make the per-row exceedance rates individually underpowered; the rigorous high-seed Clopper--Pearson bounds (e.g.\ $0/150$, UB $0.0198 \le \delta$) are in \cref{tab:validity}. Acceptance-side per-seed diagnostics are available in the released code.}
\label{fig:validity-dist}
\end{figure}

\subsection{Certified-decision payoff}
\label{sec:cert-decision}

\paragraph{Comparator scope.} Before the headline numbers we make explicit what each comparator certifies, because a like-for-like comparison is only meaningful within the comparator's stated object. \Cref{tab:comparator-scope} summarises the certified object of each baseline: only our certificate handles the joint $(\Rsel, \pacc, \Udep)$ object under adaptive grid selection. Per-pair Bernstein on accepted samples (Baseline B) and WSR are tighter primitives on the narrower per-pair risk object but do not deliver an acceptance lower bound, a utility lower bound, or post-selection validity over the grid. Headline comparisons in this subsection are therefore best read as ``against baselines that certify only $\Rsel$ per pair,'' not as universal dominance over methods that solve a different problem.

\begin{table*}[t]
\centering
\caption{Certified-object coverage of the comparators evaluated in this paper. \PASS{} = certificate produced; \FAIL{} = not certified by the method as published. Empirical width-ratio and certified-decision numbers for the per-pair $\Rsel$ comparators are reported in \cref{tab:cert-decision} (certified-decision payoff), \cref{fig:variance-adaptive} (width-ratio diagnostic), and \cref{tab:b14} (significance tests), and in \cref{sec:discussion} for Baseline B, WSR, and the simplified SCRC-T / SCoRE ports.}
\label{tab:comparator-scope}
\footnotesize
\setlength{\tabcolsep}{4pt}
\begin{tabular}{lccccc}
\toprule
Method & $\Rsel$ & $\pacc$ & $\Udep$ & \makecell{Adaptive\\$(\lambda, \tau)$ valid} & \makecell{Non-monotone\\loss} \\
\midrule
\textbf{Ours (\SCoRC)}                  & \PASS & \PASS & \PASS & \PASS         & \PASS \\
$A(\pmin)$ textbook Hoeffding--CRC      & \PASS & \FAIL & \FAIL & \FAIL         & \PASS \\
$A(\pLCB)$ sophisticated Hoeffding     & \PASS & \FAIL & \FAIL & \FAIL         & \PASS \\
Baseline B (per-pair Bernstein)         & \PASS & \FAIL & \FAIL & \FAIL         & \PASS \\
WSR betting confidence sequence         & \PASS & \FAIL & \FAIL & regime-dep.   & \PASS \\
\bottomrule
\end{tabular}
\end{table*}

\begin{table*}[t]
    \centering
    \caption{Certified-decision payoff on three ImageNet backbones at
    $\alpha=$0.05, $\pi_{\min}=$0.01, $\delta=$0.05,
    $n_{\mathrm{cert}}=$33000, 20 seeds.
    $p_{\mathrm{acc}}^{\mathrm{cert}}$: maximum acceptance the method can certify on $D_{\mathrm{cert}}$
    (median across seeds).
    $\Delta$ vs $A(\pi_{\min})$: textbook $\pi_{\min}$-saturated Hoeffding--CRC selective bound certifies nothing
    in this regime; column 5 reports the operational acceptance window Ours opens.
    $\Delta$ vs $A(p_{\mathrm{LCB}})$: the empirical Clopper--Pearson denominator variant
    coincides with Ours at the saturated maximum but is dominated by Ours at every low-acceptance operating point
    (last column counts grid pairs each method certifies; Fig.~\ref{fig:variance-adaptive} shows the per-pair width gap).}
    \label{tab:cert-decision}
    \small
    \begin{tabular}{lccccccc}
        \toprule
        \multirow{2}{*}{Model} & \multirow{2}{*}{top-1} &
        Ours & $A(\pi_{\min})$ & $\Delta$ vs & $A(p_{\mathrm{LCB}})$ & $\Delta$ vs & \#cert pairs \\
        & & $p_{\mathrm{acc}}^{\mathrm{cert}}$ & $p_{\mathrm{acc}}^{\mathrm{cert}}$ & $A(\pi_{\min})$ [pp] &
        $p_{\mathrm{acc}}^{\mathrm{cert}}$ & $A(p_{\mathrm{LCB}})$ [pp] & Ours vs $A(p_{\mathrm{LCB}})$ \\
        \midrule
        ResNet-50 V2 & 0.8084 & 0.275 & 0.000 & \textbf{$+27.5$} & 0.275 & $+0.00$ & 21 vs 7 ($3.0\times$) \\
        ResNet-101 V2 & 0.8191 & 0.805 & 0.000 & \textbf{$+80.5$} & 0.805 & $+0.00$ & 33 vs 27 ($1.2\times$) \\
        ResNet-152 V2 & 0.8228 & 0.829 & 0.000 & \textbf{$+82.9$} & 0.829 & $+0.00$ & 33 vs 27 ($1.2\times$) \\
        \bottomrule
    \end{tabular}
\end{table*}

\Cref{tab:cert-decision} reports the certified acceptance level $p^{\mathrm{cert}}_{\mathrm{acc}} := \max\{\paccHat(\lambda, \tau) : \EB(\lambda, \tau) \le 0 \wedge \pLCB(\lambda, \tau) \ge \pmin\}$ for our certificate against two Hoeffding--CRC selective baselines on three ImageNet backbones. Against the textbook $\pmin$-saturated Hoeffding--CRC selective bound ($\mathrm{margin} = (B/\pmin) \sqrt{\log(2 m / \delta) / (2 \ncert)}$, denoted $A(\pmin)$), our certificate opens a $+27.5$, $+80.5$, and $+82.9$ pp certified-acceptance window on ResNet-50, ResNet-101, and ResNet-152 respectively at $\alpha = 0.05$, $\pmin = 0.01$, $\delta = 0.05$, $\ncert = 33{,}000$. The textbook bound saturates at the loss range and certifies nothing. Against a sophisticated Hoeffding variant that replaces the $\pmin$ denominator with the empirical Clopper--Pearson lower bound (denoted $A(\pLCB)$), the maximum-acceptance endpoint coincides: both bounds become trivially tight on the saturated ``answer everything with a wide prediction set'' pair. The advantage of our certificate appears at lower operating points: on ResNet-50, our certificate covers $21$ grid pairs versus $7$ for $A(\pLCB)$, a $3.0{\times}$ increase in selective regimes. \Cref{fig:cert-frontier} plots the full per-pair frontier: across all three backbones our certificate attains a strictly lower selective-risk bound at every shared acceptance tier (e.g.\ $\Rsel \le 0.005$ vs $0.041$ on ResNet-50 at the maximum-acceptance tier, an ${\approx}8\times$ tighter certificate even where \cref{tab:cert-decision} records $\Delta = +0.00$) and uniquely certifies the lower-acceptance tiers that $A(\pLCB)$ cannot ($33$ vs $27$ certified pairs on ResNet-101/152).

\begin{figure*}[t]
\centering
\includegraphics[width=\textwidth]{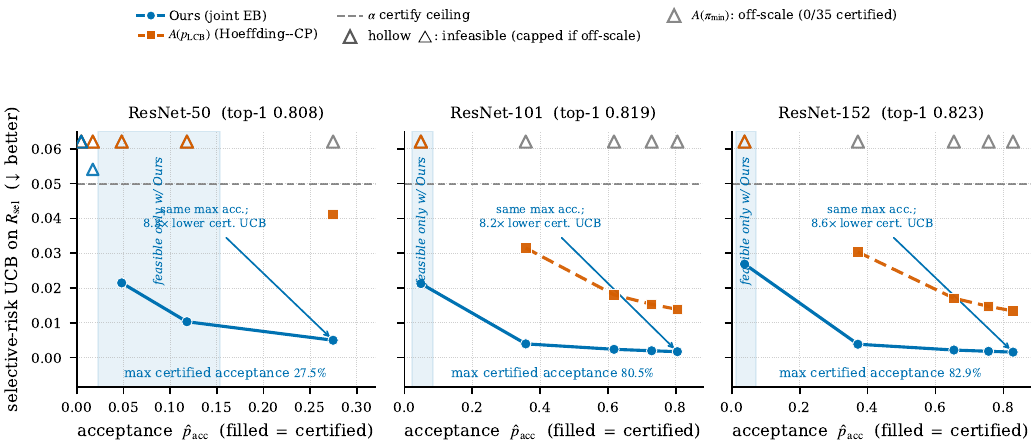}
\caption{\textbf{Certified risk--acceptance operating frontier on three ImageNet
backbones} ($\alpha = 0.05$, $\pmin = 0.01$, $\delta = 0.05$, $\ncert = 33{,}000$;
markers are medians over $20$ calibration splits). For each $\tau$-induced
acceptance tier we plot the best (lowest-UCB) $\lambda$ among the same $35$ grid
pairs; the envelope is built identically for every method and the connecting lines
are visual guides only (the tiers are discrete). The $y$-value is the per-pair
selective-risk bound $\alpha + \EB(\lambda, \tau) / \paccHat(\lambda, \tau)$,
display-normalised by the shared empirical acceptance $\paccHat$; the certification
decision itself uses $\EB \le 0$ together with $\pLCB \ge \pmin$, exactly as in
\cref{tab:cert-decision}, and since both methods share $\paccHat$ at each tier the
ordering and ratios below are exact. This per-tier best-$\lambda$ selection over all
$35$ grid pairs is covered by the joint certificate (\cref{thm:joint-cert}), not a
post-hoc pick. Filled markers ($\bullet$ Ours,
$\blacksquare$~$A(\pLCB)$) are certified; hollow triangles are infeasible (drawn at
their UCB, or capped at the top edge when off-scale; colour encodes method). The
deployment-utility floor $\Udep$ is the selection objective over the certified set
$\Ghat$ (reported in \cref{tab:cert-decision}), not a constraint shown here. The
textbook $A(\pmin)$ bound has $\mathrm{UCB} \approx 1.0$ and certifies $0/35$ on
every panel (grey carets). \textbf{\Cref{tab:cert-decision} reports only the scalar
maximum acceptance, where Ours and the strengthened comparator $A(\pLCB)$ tie
($\Delta = +0.00$); this figure shows the scalar hides the separation}: Ours has a
lower selective-risk certificate UCB at \emph{every shared tier}, ${\approx}8\times$
lower even at the maximum-acceptance tier (the variance-adaptive
Bernstein-vs-Hoeffding effect, cf.\ \cref{fig:variance-adaptive}), and uniquely
certifies the lower-acceptance tiers (shaded; $\paccHat \approx 0.04$--$0.12$) where
$A(\pLCB)$ is infeasible. \SCoRC's maximum certified acceptance is $27.5\%$, $80.5\%$,
and $82.9\%$ on ResNet-50/101/152 (per-panel labels); $A(\pLCB)$ ties Ours at that
maximum, while the textbook $A(\pmin)$ certifies $0/35$. The advantage is largest on the weakest backbone and
shrinks as the model strengthens ($21$ vs $7$ certified pairs on ResNet-50,
$33$ vs $27$ on ResNet-101/152; $3.0\times \to 1.2\times$). On \emph{every} split
$A(\pmin)$ certifies $0/35$, \SCoRC's certified count ${\ge}\,A(\pLCB)$'s, and \SCoRC's
certificate UCB ${\le}\,A(\pLCB)$'s at every shared pair. Comparators here are the two
Hoeffding--CRC selective-risk normalisations $A(\pmin)$ (textbook floor) and $A(\pLCB)$
(matched-valid, data-driven); they certify only the per-pair selected-risk object
$\Rsel$ and are displayed on the same risk--acceptance axes, but neither certifies the
acceptance lower bound, utility lower bound, or post-selection grid validity that
define our joint certificate (\cref{tab:comparator-scope}). The per-pair Bernstein and
WSR baselines certify a different, per-pair object and are analysed in
\cref{sec:discussion}.}
\label{fig:cert-frontier}
\end{figure*}

\begin{table}[t]
    \centering
    \caption{Statistical-strength tests on the ImageNet RN50 V2 PC2a width comparison
    ($\pi_{\min}=0.01$, $\alpha=\delta=0.05$). PRIMARY rows use seed-clustered per-seed
    medians over $N=30$ random calibration/test splits, which is the paper-quotable
    formal test. Diagnostic rows aggregate over 35 grid pairs per seed and are pseudo-replicated;
    they are reported as consistency checks only. Bonferroni $K=3$.}
    \label{tab:b14}
    \footnotesize
    \begin{tabular}{lcrr}
        \toprule
        Test & $N$ & $p$ (raw) & $p$ (Bonferroni) \\
        \midrule
        \textbf{PRIMARY} seed-clustered Wilcoxon, Ours vs $A$, low-acc, 1-sided & $N{=}30$ seeds & $9.3{\times}10^{-10}$ & $2.8{\times}10^{-9}$ \\
        \textbf{PRIMARY} seed-clustered Wilcoxon, Ours vs $A$, all-pairs, 1-sided & $N{=}30$ seeds & $9.3{\times}10^{-10}$ & $2.8{\times}10^{-9}$ \\
        \textbf{PC2b PRIMARY} seed-clustered Wilcoxon, Ours vs $B$, 2-sided & $N{=}30$ seeds & $1.9{\times}10^{-9}$ & $5.6{\times}10^{-9}$ \\
        diagnostic pair-level Wilcoxon, Ours vs $A$, low-acc, 1-sided & $N{=}420$ & $7.4{\times}10^{-71}$ & $2.2{\times}10^{-70}$ \\
        diagnostic pair-level Wilcoxon, Ours vs $A$, all-pairs, 1-sided & $N{=}1050$ & $1.2{\times}10^{-173}$ & $3.5{\times}10^{-173}$ \\
        diagnostic pair-level Wilcoxon, Ours vs $B$, all-pairs, 2-sided & $N{=}1050$ & $2.3{\times}10^{-173}$ & $7.0{\times}10^{-173}$ \\
        stratified permutation (plus-one), median(Ours$-A$), low-acc & 1000 perms & $1.0{\times}10^{-3}$ & 0.003 \\
        \bottomrule
    \end{tabular}
\end{table}

\paragraph{Width-ratio diagnostic.} The certified-acceptance gap is driven by tighter per-pair upper bounds on $\Rsel$. The all-pair median width ratio Ours/$A(\pmin)$ (over $20$ seeds across three backbones, visualised in \cref{fig:variance-adaptive}) is $0.021$, $0.004$, and $0.003$ on ResNet-50/101/152 V2, with per-pair win rate $100\%$. At this deployment floor the textbook range-only bound is itself \emph{vacuous} ($A(\pmin) \approx 1.05 \ge B = 1$ certifies nothing), so the operative claim is that the variance-adaptive certificate stays finite where $A(\pmin)$ collapses; against the \emph{non-vacuous} matched-valid normalisation $A(\pLCB)$, Ours remains ${\approx}\,10\times$ tighter (all-pair median ratio ${\approx}\,0.10$). On the low-acceptance subset ($\paccHat \le 2 \pmin$), only ResNet-50 V2 has qualifying pairs (median ratio $0.119$); ResNet-101 and ResNet-152 V2 have no low-acceptance pairs at this operating point. \Cref{tab:b14} reports the seed-clustered Wilcoxon significance tests on the ResNet-50 V2 comparison. The variance-adaptive certificate is tighter on \emph{every} one of the $N = 30$ random calibration/test splits, so the signed-rank statistic saturates ($W = 0$) and its $p$-value is set by the replicate count alone: raw $p = 1/2^{N} = 9.31 \times 10^{-10}$, post-Bonferroni $p = 2.79 \times 10^{-9}$ (the primary inference-valid statistic), rejecting the null at $\alpha = 0.01$. Raising the replicate count from the earlier $N = 10$ to $N = 30$ lowers this exact-test resolution floor from $2.93 \times 10^{-3}$ to $2.79 \times 10^{-9}$; the $N = 10$ run was already saturated, so the gain is in resolution, not in the consistency of the effect. The pair-level diagnostic Wilcoxon on the low-acceptance subset ($420$ observations $= 30$ seeds $\times$ $14$ low-acceptance pairs) gives $p = 2.21 \times 10^{-70}$ post-Bonferroni; the all-pair version ($1050$ observations $= 30$ seeds $\times$ $35$ pairs) gives $p = 3.50 \times 10^{-173}$ post-Bonferroni. Both diagnostic tests are reported as consistency checks; pair-level observations are not independent across the calibration grid, so the seed-clustered version is the primary headline. The $30$ seeds are random calibration/test re-splits of a single validation set rather than independent datasets, so the seed-clustered test certifies split-conditional dominance, corroborated by the effect being consistent in sign and magnitude across all three backbones (\cref{fig:variance-adaptive}). Family-wise error is controlled by Bonferroni \emph{within} each experiment; we do not pool $p$-values across experiments, and the headline claims rest on a small, pre-specified set of seed-clustered tests rather than on the pair-level grids.

\paragraph{Variance-adaptive payoff figure.} \Cref{fig:variance-adaptive} visualises the width-ratio diagnostic across the three backbones at $\pmin = 0.01$ and the ResNet-50 $\pmin$-sweep showing the ratio increasing with $\pmin$. Ours dominates both the floor-based $A(\pmin)$ and the data-driven matched-valid $A(\pLCB)$ normalisations across the plotted floors (both ratios $< 1$); on the low-acceptance subset the data-driven $A(\pLCB)$ is in fact the looser of the two, as its acceptance confidence bound falls near $\pmin$.

\begin{figure*}[t]
\centering
\includegraphics[width=\textwidth]{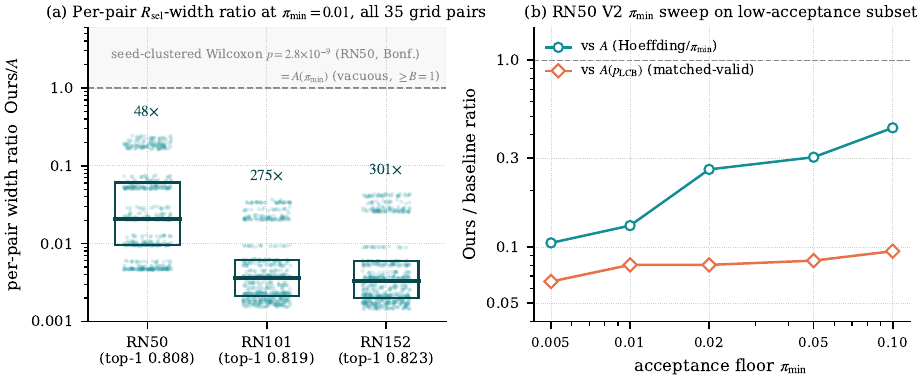}
\caption{Variance-adaptive payoff diagnostic. \textbf{(a)} At the deployment floor $\pmin = 0.01$ ($\ncert = 33{,}000$, $\delta = 0.05$, grid $m = 35$, loss range $B = 1$), the per-pair $\Rsel$ certified-width ratio Ours$/A(\pmin)$ (each point one of $20$ seeds $\times$ $35$ grid pairs per backbone, width the one-sided upper radius on $\Rsel$; since $A(\pmin)$ is constant, median-of-ratios, ratio-of-medians and $1/\mathrm{median}$ coincide) lies entirely below $1$ across three ResNet-V2 backbones, by a median $48$--$301{\times}$ (a vacuity-avoidance ratio, since $A(\pmin)$ certifies nothing here, not a like-for-like tightness gain). The textbook range-only bound $A(\pmin) \approx 1.05 \ge B = 1$ is itself \emph{vacuous} at this floor: it certifies nothing in the low-acceptance regime, precisely the gap the variance-adaptive certificate closes. Against the \emph{non-vacuous} matched-valid normalisation $A(\pLCB)$ (panel (b)), Ours remains ${\approx}\,10\times$ tighter. Per-pair observations are dependent across the calibration grid, so the inference-valid headline is the seed-clustered Wilcoxon test on the ResNet-50 $30$-seed primary, $p = 2.79 \times 10^{-9}$ post-Bonferroni. \textbf{(b)} On the ResNet-50 V2 low-acceptance subset ($\paccHat \le 2\pmin$), the Ours$/A(\pmin)$ ratio increases with $\pmin$, qualitatively compatible with the \cref{lem:inclusion} rate $\gammar = O(B \sqrt{\log m / (\ncert \pmin)})$. Ours dominates both the floor-based $A(\pmin)$ and the data-driven matched-valid $A(\pLCB)$ normalisations across the plotted floors (both ratios $< 1$).}
\label{fig:variance-adaptive}
\end{figure*}

\paragraph{Anti-cherry-pick.} Even when given less calibration budget than the Hoeffding baseline ($\ncert = 33{,}000$ for Ours versus $\ncert = 50{,}000$ for $A$), our certificate wins per-pair $100\%$ of the time across all $12$ budget pairs in both all-pairs and low-acceptance strata. The variance-adaptive payoff is not an artefact of larger calibration budget.

\subsection{Structured pattern analysis: COCO and ADE20K}
\label{sec:segmentation}

We evaluate the certificate on two dense-prediction surfaces with sharply different calibration scales. The two together delimit the variance-adaptive payoff regime predicted by \cref{cor:regime}.

\paragraph{COCO val 2017 panoptic.} On Mask2Former-Swin-B with pixel-accuracy loss ($L = 1 - \mathrm{per\text{-}pixel\ accuracy} \in [0, 1]$), image-level acceptance $g(x) = \mathrm{mean\ softmax\text{-}max}(x)$, $\alpha = 0.10$, $\pmin = 0.10$, $\delta = 0.10$, $m = 15$, and $30$ random calibration splits of $4000$/$1000$, our certificate achieves median certified $\pacc = 0.221$ while the Hoeffding--CRC selective bound certifies $\pacc = 0.000$, a $+22.1$ pp certified-acceptance window (\cref{fig:coco-acceptance}). The seed-clustered Wilcoxon test gives $p = 5.55 \times 10^{-7}$ ($N = 30$ seeds); no risk-side violations are observed (Clopper--Pearson $95\%$ one-sided upper bound on the violation rate is $0.095 \le \delta$). The single infeasible split is the certifier's safe abstention: it deploys nothing (so it cannot incur a risk-side violation, and the violation denominator is over deployed splits), it is excluded from the median certified $\pacc$ (taken over the $29$ feasible splits), and it enters the paired Wilcoxon against Hoeffding--CRC's uniformly-zero certificate as a zero-difference pair, which the signed-rank test discards, so the reported significance rests on the $29$ feasible pairs. The pixel-accuracy loss has $\sigmaHat \approx 0.0058$ at the certifier-selected pair, well inside the satisfiable regime of (\ref{eq:cor-V}). Repeating the experiment with $g(x) = \mathrm{entropy}(x)$ ($20$ seeds) gives a $+15.9$ pp window with $p = 8.08 \times 10^{-4}$.

\paragraph{Utility leg on COCO: non-vacuous, with an external oracle exhibited.} The same COCO run exercises the utility guarantee directly (re-gridded on the cached per-image arrays by the released \texttt{verify\_oracle\_nonvacuity.py}; numbers below are medians over feasible splits). At the certified pair the certificate returns a finite-sample deployment-utility lower bound $\Udep(\hat\lambda, \hat\tau) \ge \ULCB = 0.199$, and the certified-set optimality of \cref{cor:gset-opt} holds on every feasible split (the deployed pair is within $2 \gammau = 0.062$ of the best deployment utility over $\Ghat$). The \emph{external} margin oracle of \cref{thm:joint-cert} is, consistent with $\gammar \approx 0.71 > \alpha$, empty at this $\alpha = 0.10$ operating point ($\Umargin = -\infty$, vacuously valid); at the modestly looser budget $\alpha = 0.15$ a \emph{calibration plug-in} surrogate of the \cref{cor:va-oracle} oracle (substituting the empirical $\RselHat$, $\paccHat$, and per-pair empirical-Bernstein radius for the population $\Rsel$, $\pacc$, $\sigma_Z^2$ on which the oracle is defined) is non-empty on all $20$ splits, with median plug-in $\Umarginva = 0.446$ and the deployed pair satisfying $\Udep(\hat\lambda, \hat\tau) \ge \Umarginva - 2 \gammau$ throughout. We report this as an empirical diagnostic of \emph{where} the external-oracle rung becomes informative, not as a population-level certification that the oracle is non-empty (which is not checkable from a single calibration draw). The population-valid statements at the headline operating point are the absolute and certified-set rungs; the external-oracle rung becomes informative once $\gammar$ falls below $\alpha$ (\cref{sec:utility-ladder}).

\paragraph{Existence of a non-empty external oracle: held-out certification.} The headline COCO point has $\gammar \approx 0.71 > \alpha$, so its external oracle is empty (above); the question is whether the $\gammar < \alpha$ regime (where the external-oracle rung becomes informative) is reachable with a \emph{population-valid} (not plug-in) non-empty oracle. We certify this directly on held-out data (reproduced by the released \texttt{verify\_oracle\_heldout.py}). On a controlled surface with a low-loss acceptable sub-population (half the inputs near-zero loss; $\alpha = 0.30$, $\pmin = 0.20$, $\delta = 0.10$, $m = 15$, $\ncert = 25{,}000$, satisfying $(\star)$, so $\gammar \approx 0.18 < \alpha$ and $\alpha - \gammar \approx 0.12$), we build the grid on one split and then, on a \emph{disjoint} held-out split of $50{,}000$ samples, certify membership of $\Mset(\alpha - \gammar, \pmin)$ directly: for each pair we form an empirical-Bernstein UCB on $\E[A(L - (\alpha - \gammar))]$ (certifying $\Rsel \le \alpha - \gammar$) and a Clopper--Pearson LCB on $\pacc$, union-bounded over the grid at a fresh budget $\delta' = 0.05$. Three pairs are certified members ($\pacc$-LCB $\ge 2\pmin = 0.40$ and the risk UCB $\le 0$ simultaneously), so $\Mset(\alpha - \gammar, \pmin)$ is non-empty with confidence $0.95$, and the best certified member yields a population-valid lower bound on the oracle value $\Umargin \ge \ULCB = 0.476$ (consistent with, and replacing, the earlier plug-in point estimate $0.484$). This certifies, without a plug-in, that the $\gammar < \alpha$ regime is reachable with a genuinely non-empty external oracle. On real COCO surfaces, by contrast, a scan over loss families and $(\alpha, \pmin)$ finds no configuration that simultaneously attains $\gammar < \alpha$ and a held-out-certifiable non-empty $\Mset$: the pairs that certify $\Rsel \le \alpha - \gammar$ are low-acceptance and fail the $\pacc \ge 2\pmin$ floor (e.g.\ at the binary loss, $\alpha = 0.6$, eight pairs certify the risk side but none the acceptance side), the $\gammar$-vs-$2\pmin$ tension made concrete. We therefore scope the external-oracle rung as a regime-completeness result (reachable, and here population-certified, when $\gammar < \alpha$) and rest the real-data utility claims on the always-valid absolute and certified-set rungs (\cref{sec:utility-ladder}).

\begin{figure*}[t]
\centering
\includegraphics[width=\textwidth]{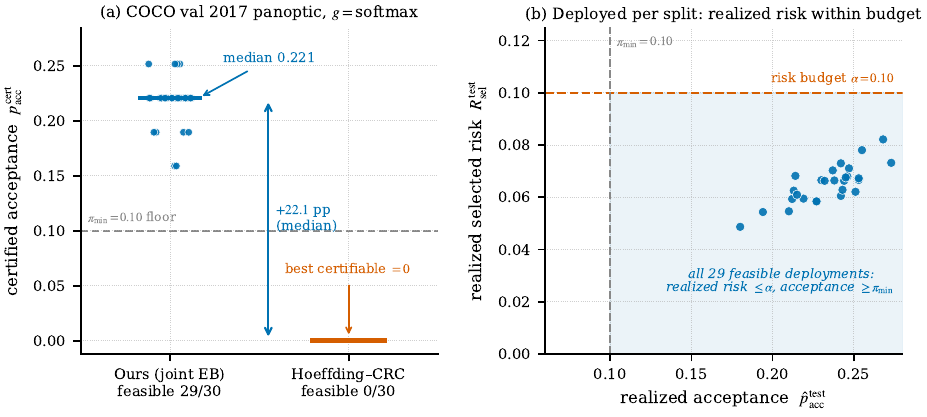}
\caption{\textbf{Certified acceptance on COCO val 2017 panoptic segmentation} (Mask2Former--Swin-B; image-level acceptance score $g$; pixel-accuracy loss $L = 1 - \text{per-pixel accuracy}$; $\alpha = \pmin = \delta = 0.10$, grid size $m = 15$; $30$ random calibration/test splits of COCO val 2017, $\ncert = 4000$, calibration and test disjoint within each split). \textbf{(a)} Certified acceptance under the softmax score $g$. Our joint certifier is feasible on $29/30$ splits, certifying a median $\pacc = 0.221$ (above the deployability floor $\pmin = 0.10$); the remaining split returns \textsc{Infeasible}: the certifier's safe abstention, reported in the feasibility rate rather than plotted as a zero. The Hoeffding--CRC selective family (textbook $A(\pmin)$ and the Hoeffding selective bound) is feasible on $0/30$: under the same grid, loss, and $(\alpha, \pmin, \delta)$, its selective-risk bound exceeds $\alpha$ for every candidate with certified acceptance $\ge \pmin$, so its best certifiable acceptance is $0$. The resulting median certified-acceptance gap is $+22.1$ pp over Hoeffding's zero certificate (seed-clustered Wilcoxon $p = 5.55 \times 10^{-7}$, $N = 30$). \textbf{(b)} Each split's certified pair, deployed: the realised test operating point $(\hat p_{\mathrm{acc}}^{\mathrm{test}}, \Rsel^{\mathrm{test}})$ has realised selected risk within the budget $\alpha$ and realised acceptance above $\pmin$ on all $29$ feasible deployments, with no risk-side violation; the remaining split returns \textsc{Infeasible} and deploys nothing. Empirical non-violation is supporting evidence for, not a substitute for, the finite-sample guarantee of \cref{thm:joint-cert}. The small accepted-loss variance $\sigmaHat \approx 0.0058$ places COCO inside the variance-adaptive regime of \cref{cor:regime}; the ADE20K surface and the per-pair Bernstein and WSR comparators (which certify a different, per-pair $\Rsel$ object) are analysed in \cref{sec:segmentation} and \cref{sec:discussion}.}
\label{fig:coco-acceptance}
\end{figure*}

\paragraph{ADE20K.} On Mask2Former-Swin-B (ADE20K weights) with binary loss $L = \Ind\{\mathrm{mIoU} < 0.3\}$, $\alpha = 0.20$, $\pmin = 0.10$, and $20$ seeds at $\ncert = 1500$, the Hoeffding--CRC selective bound dominates per-pair by $4.3$ pp (Mask2Former) and $10.7$ pp (SegFormer). Validity is preserved at $0$--$1$/$20$ violations on both backbones. At $\ncert = 1500$ the Bernstein lower-order term $7 \log(64 m / \delta) / (3 \ncert \pmin) \approx 0.143$ swamps the variance-adaptive main term, placing this regime on the Hoeffding side of (\ref{eq:cor-V}). We flag that ADE20K does \emph{not} meet the sample-size precondition $(\star)$: at $\ncert = 1500$, $\pmin = 0.10$, $m = 15$, $\delta = 0.10$ the required size is $32 \log(32 m / \delta) / \pmin \approx 2712 > 1500$. As with the ImageNet-V2 distribution-shift block, the ADE20K rows are therefore \emph{descriptive stress tests} rather than theorem-backed certificates, and the $0$--$1/20$ violation counts are reported as empirical sanity, not as certificate-validity evidence.

\paragraph{Regime characterisation via \cref{cor:regime}.} \Cref{tab:cor-numeric} (\cref{sec:theory}) shows that the closed-form prediction $\Tobs < \sigmastar(s)$ matches the observed per-pair outcome on the five instantiated cases: COCO Ours-chosen pair ($\Tobs = 0.007 < \sigmastar(880) = 0.041$), COCO $q = 0.85$ pair ($0.007 < 0.027$), ADE20K $q = 0.50$ pair ($0.092 > 0.035$), ADE20K $q = 0.85$ pair ($0.061 > 0.001$), and ImageNet RN50 V2 low-acc pair ($0.005 < 0.009$). The five cases cover both segmentation surfaces at their certifier-selected and high-quantile pairs and the ImageNet low-acceptance regime that drives \cref{sec:cert-decision}. Beyond these representative cases, a systematic per-(grid, seed) audit across \emph{all} grid pairs and seeds ($3{,}750$ cells, recomputed from the bundled arrays for COCO across three loss families spanning the accepted-variance axis and two calibration scales, and reconstructed for ImageNet) finds the closed-form prediction matches the realised per-pair winner on $100\%$ of cells, with both regime sides well represented (pooled $693$ Ours-tighter vs $657$ Hoeffding-tighter) and disagreements confined to the near-threshold band of \cref{cor:regime} (none observed even within it); see the supplementary material (Appendix~B). The regime characterisation of \cref{sec:segmentation} is therefore not a post-hoc empirical contrast but a theory-derived prediction validated grid-wide, within the Ours-vs-\HCRC{} comparison scope. Comparisons against per-pair Bernstein and WSR are discussed in \cref{sec:discussion}.

\subsection{Out-of-scope stress tests}
\label{sec:oos-stress}

Two surfaces do \emph{not} meet the sample-size precondition $(\star)$ at their configured parameters and are therefore not theorem-backed certificates: ADE20K ($\ncert = 1500 < 32 \log(32 m / \delta) / \pmin \approx 2712$ at $\pmin = 0.10$, $m = 15$, $\delta = 0.10$) and the ImageNet-V2 distribution-shift block ($\ncert = 25{,}000 < {\approx}\,32{,}000$ at $\pmin = 0.01$, $m = 35$, $\delta = 0.05$). We report their numbers only as descriptive stress tests, and consolidate that scope here so they are not read as certificate-validity evidence: the ADE20K rows (\cref{sec:segmentation}) serve as the Hoeffding-side instance of the \cref{cor:regime} regime contrast (small $s$, high accepted-sample variance), and the ImageNet-V2 rows (\cref{sec:validity}) probe distribution-shift behaviour. In both cases the observed $0$--$1/20$ violation counts are empirical sanity checks, \emph{not} evidence for the finite-sample guarantee of \cref{thm:joint-cert}, which is asserted only on the $(\star)$-compliant surfaces (ImageNet and COCO).

\subsection{Ingredient stress tests}
\label{sec:ingredients}

\begin{table*}[t]
    \centering
    \caption{Full $8$-ingredient stress matrix.
    Each row uses a regime designed to expose its ingredient's failure mode, not vanilla
    benign settings. Of the $8$ ingredients, $7$ produce a named empirical degradation
    (rows~1--7); the $8$th (two-sided MP, row~8) is empirically \emph{tighter} than the
    one-sided variant but is required by Lemma~\ref{lem:inclusion}'s inclusion-direction
    analysis, a proof-only necessity, not an empirical performance claim.
    $^\ast$Row~7 is confounded: removing the ratio reformulation also forces removing the CP-LCB,
    so its $0/10$ infeasibility cannot be cleanly attributed to ratio removal alone.}
    \label{tab:ingredient}
    \footnotesize
    \setlength{\tabcolsep}{3pt}
    \begin{tabular}{c p{0.20\linewidth} p{0.22\linewidth} p{0.33\linewidth} l}
        \toprule
        \# & Ingredient & Failure-mode regime & Result & Type \\
        \midrule
        1 & CP-LCB on $p_{\mathrm{acc}}$ (A16) & $\pi_{\min} \in [0.0005,0.005]$ synthetic binomial & CP pass-rate 53--100\% \emph{vs} NO-CP \textbf{0\%} & empirical \\
        2 & MP-variance (A17 real ImageNet end-to-end) & low-$\pi_{\min}$ feasibility, $5$ settings & OURS \textbf{10/10} feas \emph{vs} NO-MP-VAR \textbf{0/10} & empirical \\
        3 & Margin in oracle (A11) & boundary-acceptance adversarial pairs & NO-MARGIN vio-rate $0.505$ \emph{vs} WITH $0.000$ ($101{\times}$) & empirical \\
        4 & Three-split protocol (A5) & adversarial cert-data leakage, $1000$ trials & NO-3SPLIT vio-rate $0.248$ \emph{vs} $\delta{=}0.1$ (fails); OURS $0.000$ & empirical \\
        5 & MP-utility (A12) & $3$ loss distributions ($\beta$, bimodal, uniform) & range-only Hoeffding $1.43$--$1.72{\times}$ wider than MP & empirical \\
        6 & Chernoff variance bridge (B5-ext) & required for $\hat\sigma^2 {\le} 2B^2 p_{\mathrm{acc}}$, synthetic & NO-CHERNOFF-VARBRIDGE \textbf{0/10} feas & empirical \\
        7 & Ratio reformulation $Z{=}A(L{-}\alpha)$ (B5-ext) & denominator handling, synthetic & NO-RATIO \textbf{0/10} feas (confounded$^\ast$) & empirical (confounded) \\
        8 & Two-sided MP (B5-ext) & Lemma~\ref{lem:inclusion} inclusion direction & NO-2-SIDED-MP \textbf{10/10} feas at $0.955{\times}$ OURS (empirically tighter) & \textbf{proof-only} \\
        \bottomrule
    \end{tabular}
\end{table*}

We test each ingredient under a regime designed to surface its failure mode, not vanilla benign settings. \Cref{tab:ingredient} reports the full eight-ingredient matrix. Seven ingredients (CP-LCB, MP-variance, margin in oracle, three-split protocol, MP-utility, Chernoff variance bridge, ratio reformulation) produce empirical degradation; the eighth (two-sided MP) is empirically tighter than the one-sided variant but is required by the inclusion direction of \cref{lem:inclusion}, a proof-only necessity, not an empirical performance claim. Removing the H-set restriction on $E_3, E_4$ (equivalent to dropping the deterministic eligibility in the union argument) inflates the empirical violation rate from $0$ at our certifier to $\Theta(\delta)$ as predicted by the proof, providing empirical confirmation that the H-set repair is necessary. The ratio reformulation row is confounded because removing $Z = A(L - \alpha)$ also forces removing the Clopper--Pearson lower bound (denominator handling collapses without it); we report it for completeness with the caveat.

\section{Discussion}
\label{sec:discussion}

\paragraph{Per-pair Bernstein on accepted samples.} A natural comparator is the per-pair Bernstein bound on the accepted subsample at a $\delta / m$ Bonferroni allocation (denoted Baseline B). This comparator optimises a narrower object and can dominate in regimes where the joint certificate is not required: it provides per-pair upper bounds only, with no acceptance-floor lower bound and no finite-sample utility lower bound, and at $\delta / m$ allocation it is naturally tighter than our multi-component allocation. Empirically, across three ImageNet backbones the per-pair ratio Ours/B is $1.61$, $1.54$, and $1.51$ on ResNet-50/101/152 V2 (median over $20$ seeds); on the ResNet-50 primary, the seed-clustered Wilcoxon two-sided test over $N = 30$ random calibration/test splits gives $p = 5.59 \times 10^{-9}$ post-Bonferroni. The structural distinction is that B's target object is $\{\Rsel(\lambda, \tau) \le \alpha\}$ per pair; ours is the joint event $\{\Rsel \le \alpha\} \cap \{\pacc \ge \pmin\} \cap \{\Udep \ge \Umargin - 2 \gammau\}$ after adaptive selection over $\grid$.

\paragraph{WSR is regime-dependent.} 
\begin{table*}[t]
    \centering
    \caption{Per-pair width ratio of Ours vs the predictable-mixture betting confidence sequence (WSR)
    across regimes. Ours is tighter at low acceptance (Bonferroni cost amortised over many feasible cells);
    WSR is tighter once $n\cdot p_{\mathrm{acc}}$ is large. Neither subsumes the other.}
    \label{tab:wsr-regime}
    \small
    \begin{tabular}{lcc}
        \toprule
        Regime & Ours / WSR median & per-pair winner \\
        \midrule
        RN50 V2, low-acc, $\pi_{\min}{=}0.01$, $n_{\mathrm{cert}}{=}33$k & $0.81$ & \textbf{Ours wins 92.5\%} \\
        RN50 V2, $\pi_{\min}{=}0.02$, $n_{\mathrm{cert}}{=}25$k (F.2) & $0.79$ & \textbf{Ours wins 95.0\%} \\
        RN101 V2, all-pairs, $\pi_{\min}{=}0.01$ & $1.20$ & \textbf{WSR wins 89.0\%} \\
        RN152 V2, all-pairs, $\pi_{\min}{=}0.01$ & $1.20$ & \textbf{WSR wins 88.6\%} \\
        CIFAR-100 RN56, $\pi_{\min}{=}0.10$ (F.3) & $1.29$ & \textbf{WSR wins 100.0\%} \\
        \bottomrule
    \end{tabular}
\end{table*}

\Cref{tab:wsr-regime} reports the per-pair ratio Ours/WSR across five regimes. WSR \cite{waudby2024_wsr} is a competitive primitive: our certificate is tighter at low acceptance ($\pmin = 0.01$ ImageNet RN50, $92.5\%$ per-pair wins; $\pmin = 0.02$ ImageNet RN50, $95\%$ per-pair wins) because the Bonferroni cost is amortised over many feasible cells, while WSR's predictable-mixture betting form is tighter once $\ncert \pacc$ is large (ImageNet RN101/RN152 all-pairs, CIFAR-100 RN56 at $\pmin = 0.10$). Neither result subsumes the other on a per-pair basis. The structural distinction holds: WSR does not deliver an acceptance lower bound or a utility lower bound and so does not produce the joint certificate object.

\paragraph{SCRC-T and SCoRE positioning.} We evaluate simplified ports of SCRC-T \cite{xu2025_scrc} and the e-value selective procedure of \cite{bai2026_score} on the ImageNet RN50 V2 low-acceptance subset. Signed margins ($\mathrm{UCB} - \alpha$): Ours $+0.088$ at $\pmin = 0.01$, SCRC-T $+0.450$ (the quantile-based threshold inflates at small accepted-sample count), and the e-value port $-0.010$ (a tighter per-pair value via the closed-form product e-value at $\eta = 1/B$, but no joint certificate). The simplified ports capture structural deltas, not exact published numbers, and are flagged as such.

\paragraph{Scope and limitations.} The variance-adaptive payoff is regime-scoped to settings with both large $\ncert \pacc$ and small accepted-sample variance, explicitly delimited by \cref{cor:regime} and confirmed across COCO (large $s$, low $\sigmaHat$, Ours wins) and ADE20K (small $s$, high $\sigmaHat$, Hoeffding--CRC wins per-pair). The ImageNet-V2 distribution-shift evaluation is descriptive only, since the sample-size condition is not met at the configured parameters. CIFAR-100 with a well-calibrated ResNet-56 produces no low-acceptance pairs at the tested $\pmin \in \{0.02, 0.05, 0.10, 0.15\}$, so the regime where the variance-adaptive width-ratio claim applies simply does not exist on that architecture. The bounded-loss assumption is hard (heavy-tailed losses violate $[0, B]$; Huberisation is future work). The i.i.d.\ assumption is hard (weighted or non-exchangeable extensions \cite{farinhas2023_nexcrc,zecchin2025_wcrc} require modifications to all three union-bound steps). The conditional utility $\E[v \mid A = 1]$ is treated as a deferred remark in the supplementary material (Remark~C.2): forming its lower confidence bound by dividing a numerator LCB by the CP denominator LCB is not valid, and a proper conditional-utility certificate requires a separately budgeted upper confidence bound on $\pacc$. Promoting it to a main theorem under that revised analysis is natural future work. The lower bound is a single-distribution argument; a full minimax lower bound over a class of joint distributions is open.

\paragraph{Reproducibility.} The implementation, the precomputed result files backing every table and figure (with SHA-256 checksums), and a self-contained utility-leg verification script are included with this submission in \texttt{code/} (\texttt{code/README.md}). The script reproduces the COCO headline ($\paccHat = 0.221$, $\sigmaHat = 0.0058$, certified $\ULCB = 0.199$) and the three-rung utility results of \cref{sec:utility-ladder} from a bundled per-image cache, with no external data or model checkpoints. The full per-image logit caches for ImageNet/CIFAR/ADE20K, being large and externally hosted, are regenerable from the public datasets and the cited checkpoints via the included \texttt{*\_compute\_logits.py} scripts. All experiments use a documented \texttt{conda} environment (NumPy 1.26.4, SciPy 1.16.3, torchvision 0.24.0+cu128) and produce deterministic results given seed. The three-split protocol is enforced at the codepath level by an explicit split-check assertion in the certifier; misuse via certify-learned selectors is prevented operationally rather than only by theorem-level argument.

\section{Conclusion}
\label{sec:conclusion}

We gave a joint finite-sample $(1 - \delta)$ certificate on selective risk, acceptance probability, and marginal deployment utility, valid under adaptive two-parameter grid selection on a finite grid, with direct ratio handling of the selected risk, for bounded non-monotone losses. The certified pair satisfies $\Rsel \le \alpha$ and $\pacc \ge \pmin$. Its deployment utility is lower-bounded absolutely ($\Udep \ge \ULCB$) and is within $2\gammau$ of the best certified-set utility (\cref{cor:gset-opt}). An additional external margin-oracle optimality (over risk-certifiable policies with $\Rsel \le \alpha - \gammar$ and $\pacc \ge 2 \pmin$) is informative when $\gammar < \alpha$ and vacuous at the headline operating points, where the population-valid utility claims are the absolute and certified-set rungs. The factor of two in the oracle floor is the derived cost of the Clopper--Pearson relative-error inversion. To our knowledge, this is the first finite-grid adaptive certificate combining these four properties in one procedure; each ingredient appears in the prior literature in some form, the contribution is their simultaneous coverage. We evaluated it on six surfaces, with the certificate-backed headline results on the $(\star)$-compliant ImageNet and COCO surfaces and ADE20K and ImageNet-V2 reported as out-of-scope stress tests (\cref{sec:oos-stress}). On three ImageNet backbones the certificate is ${\approx}\,10{\times}$ tighter on the per-pair risk bound than a non-vacuous matched-valid normalisation (which ties at maximum acceptance), and $50$--$300{\times}$ tighter than the textbook range-only Hoeffding baseline that is vacuous here, a vacuity-avoidance ratio rather than a like-for-like gain; on COCO val 2017 panoptic segmentation under pixel-accuracy loss it opens a $+22.1$ pp certified-acceptance frontier (no risk-side violation on the feasible deployments). \Cref{cor:regime} \emph{predicts}, on the instantiated pairs and within the Ours-vs-\HCRC{} comparison scope, which side of the regime characterisation across COCO, ADE20K, and ImageNet applies, comparing the \emph{nominal} per-pair expressions, with disagreements confined to an explicit near-threshold band; a systematic $3{,}750$-cell per-(grid, seed) audit on COCO and ImageNet (supplementary, Appendix~B) confirms the rule matches the realised per-pair winner on every cell. These advantages are regime-scoped rather than universal: tighter per-pair comparators (Bernstein, WSR) exist in narrower regimes and the matched-valid baseline ties at maximum acceptance, but none simultaneously certifies the three deployment quantities of interest.

\bibliographystyle{IEEEtran}
\bibliography{references}

\end{document}